\documentclass[journal]{IEEEtran} 

\usepackage{epsfig}
\usepackage{graphicx}
\usepackage{amsmath}
\usepackage{amssymb}
\usepackage{booktabs}
\usepackage{color}
\usepackage{arydshln}
\usepackage{rotating}

\usepackage[small]{caption}
\usepackage{algpseudocode}
\usepackage{pseudocode}
\usepackage{algorithm}
\usepackage{multirow}
\usepackage[square,sort,comma,numbers]{natbib}
\usepackage[hyphens]{url}
\usepackage{amsfonts}

\usepackage{amsthm}
\usepackage{mathtools}
\usepackage{subcaption}
\usepackage{caption}
\usepackage{arydshln}
\usepackage{sidecap}
\newtheorem{prop}{Proposition}

\newcommand{\rev}[1]{{\color{black}{#1}}}

\newcommand{\etal}{\mbox{\emph{et al.\ }}}

\newtheorem{theorem}{Theorem}[section]

\newtheorem{lemma}[theorem]{Lemma}
\newtheorem*{assumption}{Assumption}
\theoremstyle{definition}

\theoremstyle{remark}

\makeatletter
\renewcommand\footnoterule{%
  \kern-3\p@
  \hrule\@width \columnwidth
  \kern2.6\p@}
\makeatother

\begin{document}
\title{Cost-Sensitive Learning of Deep Feature Representations from Imbalanced Data}

\author{S.~H.~Khan,
	   ~M.~Hayat,%
       ~M.~Bennamoun,
       ~F.~Sohel%
       ~and~R.~Togneri
\IEEEcompsocitemizethanks{
\IEEEcompsocthanksitem 
S. H. Khan is with Data61, Commonwealth Scientific and Industrial Research Organization (CSIRO) and College of Engineering \& Computer Science, Australian National University, Canberra, ACT 0200, Australia.\protect\
E-mail: salman.khan @data61.csiro.au
\IEEEcompsocthanksitem 
 M. Bennamoun is with the School of Computer Science and Software Engineering, The University of Western Australia, 35 Stirling Highway, Crawley, WA 6009, Australia.\protect\\
E-mail: mohammed.bennamoun @uwa.edu.au
\IEEEcompsocthanksitem M. Hayat is with Human-Centered Technology Research Centre, University of Canberra, Bruce, Australia. \protect\\
Email: munawar.hayat@canberra.edu.au
\IEEEcompsocthanksitem R. Togneri is with the School of Electrical, Electronic and
Computer Engineering, The University of Western Australia,
35 Stirling Highway, Crawley, WA 6009, Australia.\protect\\
E-mail: roberto.togneri@uwa.edu.au
\IEEEcompsocthanksitem F. Sohel is with the School of Engineering and Information Technology, Murdoch University, 90 South St, Murdoch WA 6150, Australia.
\protect\\
E-mail: f.sohel@murdoch.edu.au}
}
\maketitle

\markboth{Journal of \LaTeX\ Class Files,~Vol.~6, No.~1, July~2015}%
{Shell \MakeLowercase{\textit{et al.}}: Bare Demo of IEEEtran.cls for Computer Society Journals}

{%
\begin{abstract}
Class imbalance is a common problem in the case of real-world object detection and classification tasks. Data of some classes is abundant making them an over-represented majority, and data of other classes is scarce, making them an under-represented minority. 
This imbalance makes it challenging for a classifier to appropriately learn the discriminating boundaries of the majority and minority classes.
In this work, we propose a \rev{cost-sensitive} deep neural network which can automatically learn robust feature representations for both the majority and minority classes.
During training, our learning procedure {jointly optimizes} the \rev{class-dependent} costs and the neural network parameters.
The proposed approach is applicable to both binary and multi-class problems without any modification. 
Moreover, as opposed to data level approaches, we do not alter the original data distribution which results in a lower computational cost during the training process.
We report the results of our experiments on six major image classification datasets and show that the proposed approach significantly outperforms the baseline {algorithms. Comparisons with popular data sampling techniques and \rev{cost-sensitive} classifiers demonstrate the} superior performance of our proposed method.
\end{abstract}

 }


\section{Introduction}
In most real-world classification problems, the collected data follows a long tail distribution i.e., data for few object classes is abundant while data for others is scarce.
This behaviour is termed the \emph{`class-imbalance problem'} and it is inherently manifested in nearly all of the collected {image classification databases} {(e.g., Fig. 1).
A multi-class dataset is said to be} \emph{`imbalanced'} or \emph{`skewed'} if some of {its (minority) classes,  in the training set, are heavily under-represented compared to other (majority) classes}. 
This skewed distribution of class instances forces the classification algorithms to be biased towards the majority classes. 
As a result, the {characteristics of the} minority classes are not adequately learned.

The class imbalance problem is of particular {interest} in \rev{real-world} scenarios, where it is {essential to correctly classify} examples from an infrequent but important minority class. 
{For instance, a particular cancerous lesion (e.g., a melanoma) which appears rarely during dermoscopy should not be mis-classified as benign (see Sec. \ref{sec:exp})}. Similarly, for a continuous surveillance task, a {dangerous activity which occurs occasionally should still be detected by the monitoring system}.
{The same applies to many other application domains, e.g., object classification, where the correct classification of a minority class sample is equally important to the correct classification of a majority class sample}.
{It is therefore required to enhance the overall accuracy of the system without unduly sacrificing the precision of any of the majority or minority classes}.
 {Most of the} classification algorithms try to minimize the overall classification error during the training process.
 {They, therefore,} implicitly assign an  identical misclassification cost to all types of errors assuming their {equivalent importance}. 
As a result the classifier tends to {correctly classify and favour} the more frequent classes.

{Despite the pertinence of the class imbalance problem to  practical computer vision, there have been very few research works on this topic in {the} recent years. 
{Class} imbalance is avoided in nearly all competitive datasets during the evaluation and training procedures (see Fig.~\ref{fig:intro}). }
{For instance, for the case of the} popular image classification datasets (such as CIFAR$-10/100$, ImageNet, Caltech$-101/256$, {and} MIT$-67$), efforts have been made by the collectors to {ensure that,} either all of the classes have a minimum representation {with} sufficient data{, or that} the experimental protocols are reshaped to use an equal number of images for all classes during the training and testing {processes} \cite{lin2014network, chatfield2014return,lee2015deeply}.
This approach {is reasonable in the case of datasets with only few classes, which have an equal probability} to appear in practical scenarios (e.g., digits in MNIST).
{However, with the increasing number of classes in the collected object datasets, it is becoming impractical to provide equal representations for all classes in the training and testing subsets. }
{For example, for a fine-grained coral categorization dataset,  endangered coral species have a significantly lower representation compared {to the more} abundant ones \cite{beijbom2012automated}.} 

\begin{figure*}
\centering

\begin{subfigure}[t]{0.48\columnwidth}
\includegraphics[width=\textwidth]{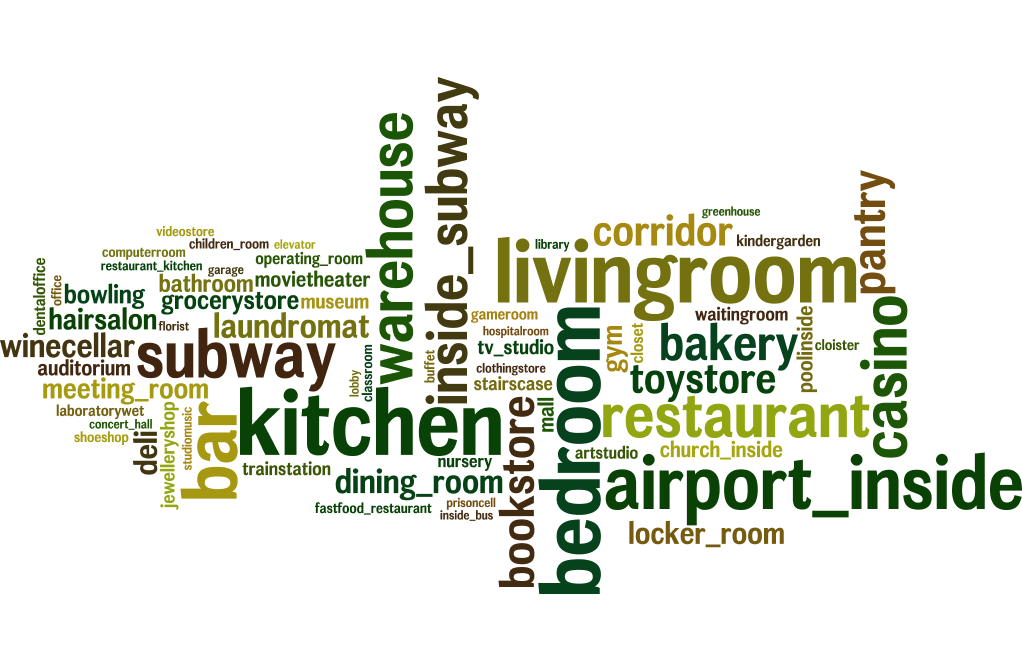}
\caption{Word cloud: MIT-67}
\label{fig:mit_word}
\end{subfigure}
\;
\begin{subfigure}[t]{0.48\columnwidth}
\includegraphics[width=\columnwidth]{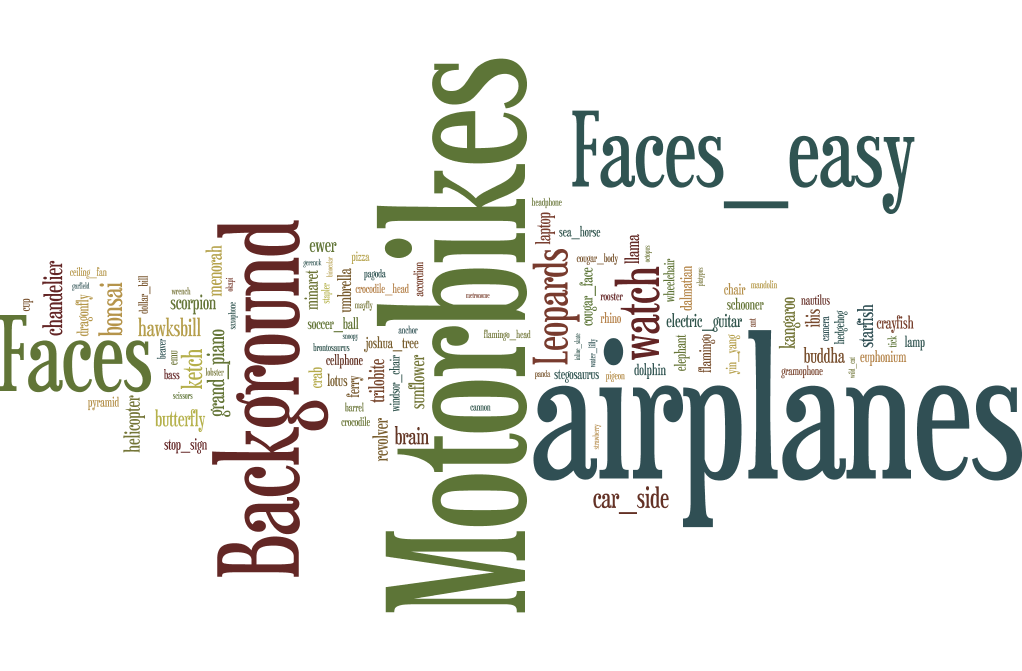}
\caption{Word cloud: Caltech-101}
\label{fig:caltech_word}
\end{subfigure}
\;
\begin{subfigure}[t]{0.48\columnwidth}
\includegraphics[width=\columnwidth]{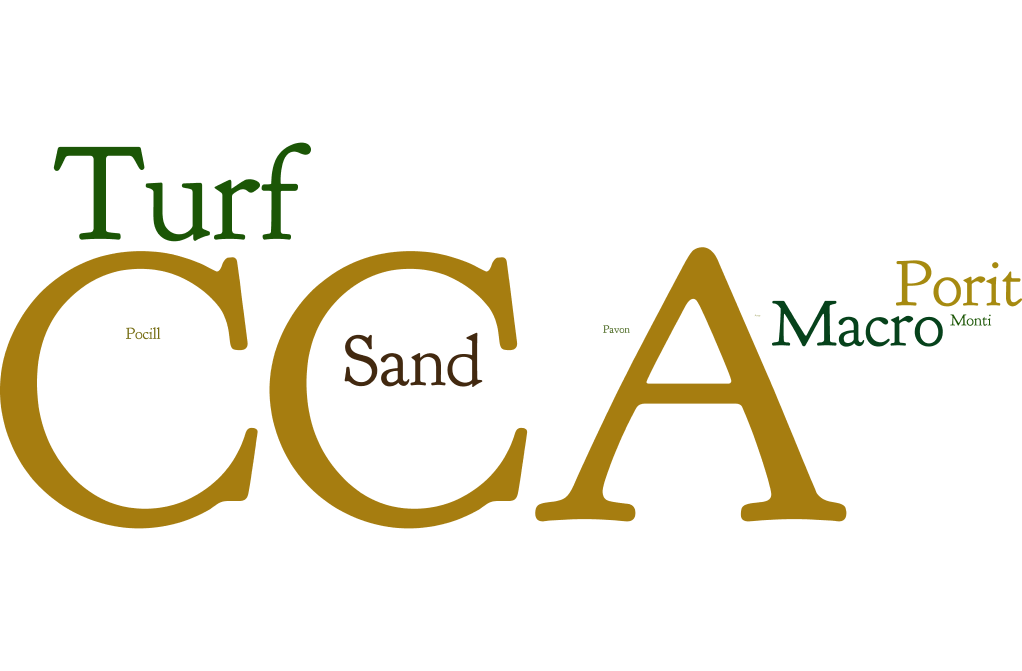}
\caption{Word cloud: MLC}
\label{fig:MLC_word}
\end{subfigure}
\;
\begin{subfigure}[t]{0.48\columnwidth}
\includegraphics[width=\columnwidth]{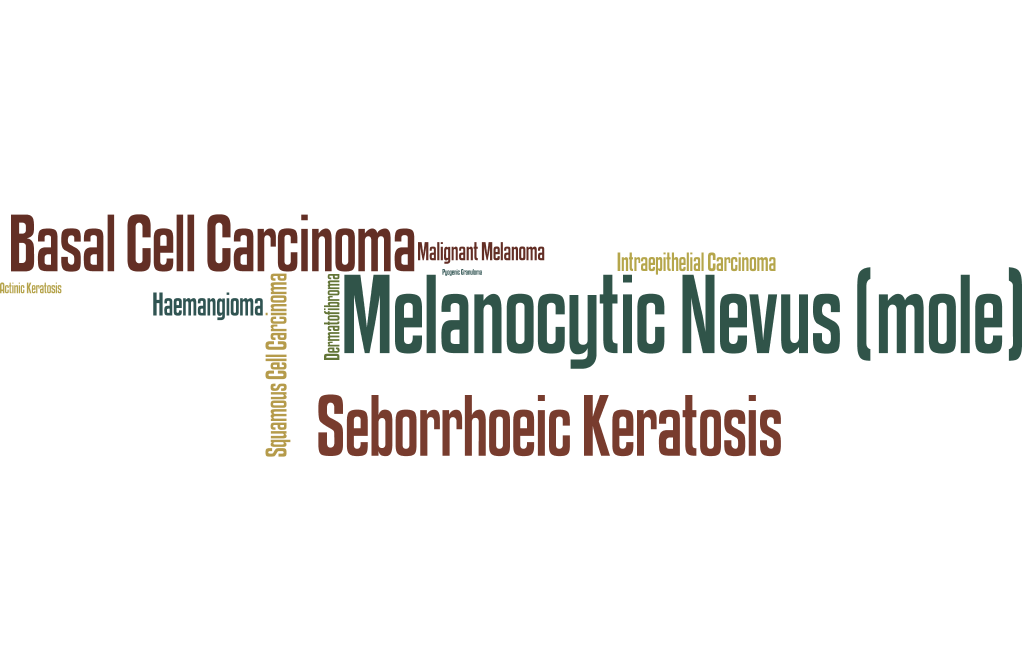}
\caption{Word cloud: DIL}
\label{fig:DIL_word}
\end{subfigure}
\\  
\begin{subfigure}[t]{0.48\columnwidth}
\includegraphics[width=\textwidth]{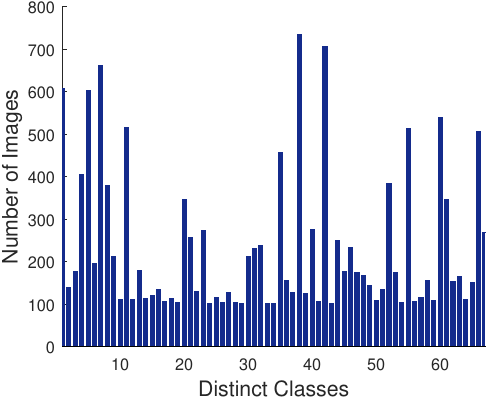}
\caption{Class frequencies}
\label{fig:mit_freq}
\end{subfigure}
\;
\begin{subfigure}[t]{0.48\columnwidth}
\includegraphics[width=\columnwidth]{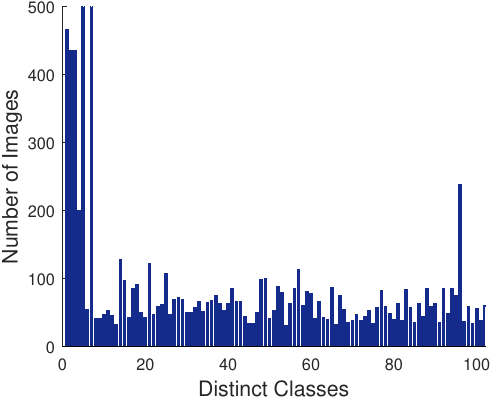}
\caption{Class frequencies}
\label{fig:caltech_freq}
\end{subfigure}
\;
\begin{subfigure}[t]{0.48\columnwidth}
\includegraphics[width=\columnwidth]{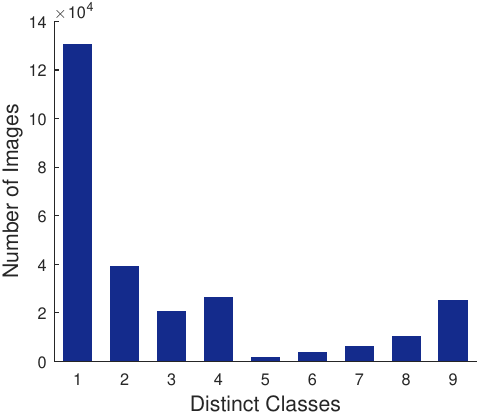}
\caption{Class frequencies}
\label{fig:MLC_freq}
\end{subfigure}
\;
\begin{subfigure}[t]{0.48\columnwidth}
\includegraphics[width=\columnwidth]{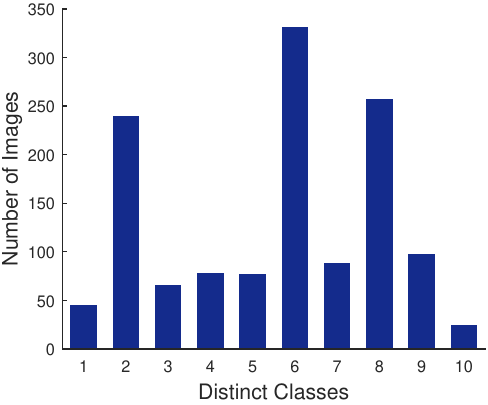}
\caption{Class frequencies}
\label{fig:DIL_freq}
\end{subfigure}
\caption{ Examples of popular classification datasets in which the number of images vary sharply across different classes. This class imbalance poses a challenge for classification and representation learning algorithms. }
\label{fig:intro} \vspace{-0.2cm}
\end{figure*}

In this work, we propose to jointly learn {robust} feature representations and classifier parameters, under a cost-sensitive setting. 
This {enables} us to learn not only an improved classifier that deals with the class imbalance problem, but also to extract suitably adapted intermediate feature representations from a deep Convolutional Neural Network (CNN).
{In this manner, we directly modify the learning procedure to incorporate \rev{class-dependent} costs during training. In contrast, previous works (such as \cite{chawla2002smote, khan2014automatic, zhang2013cost, ramentol2012smote}) only readjust the training data distribution to learn better classifiers. }
%
Moreover, unlike the methods {in e.g.,} \cite{zhou2006training, beijbom2012automated}, we do not use a  {handcrafted} cost matrix whose design is based on expert judgement and {turns into a tedious task for a large number of classes}. 
In our case, the \rev{class-dependent} costs are automatically set using {data statistics} (e.g., data distribution and separability {measures}) during the learning procedure.
%
Another major difference with {existing} techniques is that {our} class specific costs are only used during the training process and once the optimal CNN parameters are learnt, predictions can be made without any modification to the trained network. 
%
{From this perspective,}
our approach can be understood as a perturbation method, which forces the training  algorithm to learn more discriminative features. 
{Nonetheless,
it is clearly}
different from the common perturbation mechanisms used during training e.g., data distortions \cite{khan2015discriminative}, corrupted features \cite{maaten2013learning}, affine transformations \cite{krizhevsky2012imagenet} and activation dropout \cite{srivastava2014dropout}.

{Our contribution consists of the following:}
{\bf 1--} We introduce cost-sensitive versions of three widely used loss functions for joint cost-sensitive learning of features and classifier parameters in the CNN (Sec.~\ref{sec:CSL}). 
 We also show that the improved loss functions have desirable properties such as classification calibration and guess-aversion.
{\bf 2--} We analyse the effect of these modified loss functions on the back-propagation algorithm by deriving relations for propagated gradients (Sec.~\ref{sec:BPeff}).
{\bf 3--} We propose an algorithm for joint alternate optimization of the network parameters and the class-sensitive costs (Sec.~\ref{sec:learning}).
The proposed algorithm can automatically work for both binary and multi-class classification problems.
 We also show that the introduction of class-sensitive costs does not significantly affect the training and testing time of the original network (Sec.~\ref{sec:exp}).
{\bf 4--} The proposed approach has been extensively tested on six major classification datasets and has shown to outperform baseline procedures and state-of-the-art approaches (Sec.~\ref{subsec:results}). 

The remainder of this paper is organized as follows.  
We briefly discuss the related work in the next section. In Sec. \ref{sec:prob_for2} and \ref{subsec:propCM}, we introduce our proposed approach and analyse the modified loss functions in Sec.~\ref{sec:CSL}. The learning algorithm is then described in Sec.~\ref{sec:learning} and the CNN implementation details are provided in Sec.~\ref{sec:cnn}. Experiments and results are summarized in Sec.~\ref{sec:exp} and the paper concludes in Sec.~\ref{sec:conclusion}.
%
\section{Related Work}\label{sec:rel_work}

Previous research on the class imbalance problem has concentrated mainly on two levels: the data level and the algorithmic level \cite{garcia2007class}.
Below, we briefly discuss the different research efforts 
{that} tackle the class imbalance problem. 

{\bf Data level approaches} manipulate the class representations in the original dataset  by either over-sampling the minority classes or under-sampling the majority classes to make the resulting data distribution balanced \cite{garcia2007class}. 
However, these techniques change the original distribution of the data and 
{consequently} introduce 
drawbacks. 
While under-sampling can 
potentially {lose useful information about the} majority class data,  over-sampling makes the training computationally burdensome by artificially increasing the size of the training set. 
Furthermore, over-sampling is prone to cause over-fitting, when exact copies of the minority class are replicated randomly \cite{chawla2002smote, garcia2007class}. 

To address the over-fitting problem, Chawla \etal \cite{chawla2002smote} introduced a method, called SMOTE, to generate new instances by linear interpolation between closely lying minority class samples. 
These synthetically generated minority class instances may lie 
inside the convex hull of the majority class instances, a phenomenon known as \emph{over-generalization}.  
Over the years, several variants of {the} SMOTE algorithm have been proposed to solve this problem \cite{wang2014hybrid}. 
For example, Borderline SMOTE \cite{han2005borderline} only over-samples the minority class samples which lie close to the 
{class} boundaries. 
Safe-level SMOTE \cite{bunkhumpornpat2009safe} carefully generates synthetic samples in the so called \emph{safe-regions}, where the majority and minority class regions are not overlapping.
{The} local neighborhood SMOTE \cite{maciejewski2011local} considers the neighboring majority class samples when generating synthetic minority class samples and reports a better performance compared to the former variants of SMOTE.
{The} combination of  under and over sampling procedures (e.g., \cite{batista2004study,jeatrakul2010classification, ramentol2012smote}) to balance the training data have also 
shown to perform well.
{However, a drawback of these approaches is the increased computational \rev{cost that is} required for data pre-processing} and for the learning of a classification model.

{\bf Algorithm level approaches} directly modify the learning procedure to improve the 
{sensitivity}
of the classifier towards minority classes. 
Zhang \etal \cite{zhang2013cost} first divided the data into smaller balanced subsets, followed by intelligent sampling and a cost-sensitive SVM learning to deal with the imbalance problem. 
A neuro-fuzzy modeling procedure was introduced in \cite{gao2014construction} to perform leave-one-out cross-validation on imbalanced datasets. 
A scaling kernel along-with the standard SVM was used in \cite{zhang2014imbalanced} to improve the 
{generalization ability}
of learned classifiers 
{for}
skewed datasets. 
Li \etal \cite{li2014boosting} gave more importance to the minority class samples by  setting weights with Adaboost during the training of an extreme learning machine (ELM).
An ensemble of soft-margin SVMs was formed via boosting to perform well on both majority and minority classes \cite{wang2010boosting}.
These previous works hint towards the use of distinct costs for different training examples to improve the performance of the learning algorithm.
However, they do not address 
{the}
class imbalance learning of CNNs, which have recently emerged as the most popular tool for supervised classification, recognition and segmentation problems in computer vision \cite{zhang2014imbalanced, wu2003class, wu2005kba, krizhevsky2012imagenet,girshick2014rich}.
Furthermore,
they are mostly limited to the binary class problems \cite{huang2004learning,wang2010boosting}, do not perform joint feature and classifier learning and do not explore computer vision tasks which inherently have imbalanced class distributions.
\rev{In the context of neural networks, Kukar and Kononenko \cite{kukar1998cost} showed that the incorporation of costs in the error function improves performance. However, their costs are randomly chosen in multiple runs of the network and remain fixed during the learning process in each run.}
In contrast, this paper presents the first attempt 
{to incorporate automatic cost-sensitive learning in deep neural networks for imbalanced data}.

\rev{After the submission of this work for review, we note that a number of new approaches have recently been proposed to incorporate class-specific costs in the deep networks \cite{chung2015cost,wang2016training,raj2016towards}. 
 Chung \etal \cite{chung2015cost} proposed a new cost-sensitive loss function which replaces traditional soft-max with a regression loss. In contrast, this work extends the traditionally used cost-functions in CNN for the cost-sensitive setting. Wang \etal \cite{wang2016training} and Raj \etal  \cite{raj2016towards} proposed a loss function which gives equal importance to mistakes in the minority and majority classes. Different to these works, our method is more flexible because it automatically learns the balanced error function depending on the end problem. }

\section{Proposed Approach}\label{sec:prob_for}
\subsection{{Problem Formulation for  \rev{Cost-Sensitive} Classification}}\label{sec:prob_for2}
{
Let the cost $\xi'_{p,q}$ be used to denote the misclassification cost of classifying an instance belonging to a class $p$ into a different class $q$.
The diagonal of $\xi'$ (i.e., $\xi'_{p,p}, \forall p$) represents the benefit or utility for a correct prediction.
Given an input instance $\mathbf{x}$ and the cost matrix $\xi'$, the classifier seeks to minimize the expected risk $\mathcal{R}(p|\mathbf{x})$, where $p$ is the class prediction made by the classifier.
The expected risk can be expressed as:
 $$ \mathcal{R}(p|\mathbf{x}) = \sum\limits_{q} \xi'_{p,q}P(q|\mathbf{x}) , $$
 where, $P(q|\mathbf{x})$ is the posterior probability over all possible classes given an instance $\mathbf{x}$.
According to the Bayes decision theory, an ideal classifier will give a decision in favour of the class ($p^*$) with the minimum expected risk: 
\begin{align}\label{eq:ideal_clsf}
 p^* = \underset{p}{\operatorname{argmin}} \; \mathcal{R}(p|\mathbf{x}) = \underset{p}{\operatorname{argmin}} \; \mathbb{E}_{X,D}[\xi']
\end{align}
where, $X$ and $D$ define the input and output spaces respectively. 
Since, $P(q|\mathbf{x})$ cannot be found trivially, we make use of empirical distribution derived from the training data. Given a training dataset consisting of tuples comprising of data and label, $\mathcal{D} = \{\mathbf{x}^{(i)}, \mathbf{d}^{(i)}\}_M$ \rev{where $\mathbf{d}\in \mathbb{R}^N$}, we can define the empirical risk as follows:
\begin{align}
\hat{\mathcal{R}}_{\ell}(\mathbf{o}) = \mathbb{E}_{X,D}[\ell] = \frac{1}{M} \sum\limits_{i=1}^{M} \ell(\xi', \mathbf{d}^{(i)}, \mathbf{o}^{(i)}), 
\end{align}
where, $M$ is the total number of images, $\mathbf{o}^{(i)}\in \mathbb{R}^N$ is the neural network output for the $i^{th}$ sample and $\ell(\cdot)$ is the misclassification error ($0-1$ loss) or a surrogate loss function which is typically used during the classifier training. 
For the case of cost-insensitive $0-1$ loss, $\ell(\xi', \mathbf{d}^{(i)}, \mathbf{o}^{(i)}) =  \mathbb{I}(\mathbf{d}^{(i)}\neq \mathbf{o}^{(i)})$ and $\xi'$ is an $N\times N$ matrix, where $\xi'_{p,p} = 0$, and $\xi'_{p,q} = 1, \, \forall p\neq q$. \rev{Next, we briefly describe the properties of traditional used cost matrix $\xi'$, before introducing the proposed cost matrix.}

\subsubsection*{\textbf{Properties of the Cost Matrix $\mathbf{\xi'}$ }}
{Lemmas \ref{lem:zero} and \ref{lem:one} describe the main} properties of the cost matrix $\xi'$. Their proof can be found in Appendix~A
 (supplementary material).
\begin{lemma}\label{lem:zero}
Offsetting the columns of the cost matrix $\xi'$ by any constant `$c$' does not affect the associated classification risk $\mathcal{R}$.  
\end{lemma}
\noindent
For convenience, the {utility vector (i.e., the diagonal of the cost matrix)} for correct classification is usually set to zero with the help of {the property from Lemma~\ref{lem:zero}}. 
{We also show next} that even when the utility is not zero, it must satisfy the following condition:
\begin{lemma}\label{lem:one}
The cost of the true class  should be less than the mean cost of all misclassifications.
\end{lemma}
\noindent
{Finally,} using Lemmas~\ref{lem:zero} and \ref{lem:one}, we assume {the following}:
\begin{assumption}
All costs are non-negative i.e., $\xi' \succeq 0$.
\end{assumption}

{The entries of a
 traditional cost matrix (defined according to the properties above) 
usually have the form} of: 
\begin{align}
\xi' = \left\{ 
	\begin{array}{c c}
		\xi'_{p,q} = 0  &  p=q\\
		\xi'_{p,q} \in  \mathbb{N} & p\neq q. \\
	\end{array} 
			\right. 
\end{align} 
Such 
cost matrix can potentially increase the 
corresponding loss to 
{a large} value. 
 During the CNN training, this network loss can make the training process unstable and can lead to {the} non-convergence of the error function. This requires the introduction of an alternative cost matrix.
 \subsection{{Our Proposed Cost Matrix}}\label{subsec:propCM}
{We propose a new cost matrix $\xi$, 
 which is suitable for CNN training. The  cost matrix $\xi$ is 
 used to modify the output 
{of} the last layer of {a} CNN (before the softmax and the loss layer) }
 (Fig.~\ref{fig:overview_fig}). 
 The resulting activations are then squashed 
 {between} $[0,1]$ before the computation of the classification loss.

For the case of a CNN, the classification decision is made in favour of the class 
{with the maximum classification score.}
During {the} training {process,} the classifier weights are modified in order to reshape the classifier confidences {(class probabilities)} such that the desired class has {the} maximum score and the 
{other} classes have {a} considerably lower score.  
However, since the less frequent classes are under-represented in the training set, we introduce new `score-level costs' to encourage the correct classification of infrequent classes. 
Therefore the CNN outputs ($\mathbf{o}$) are modified using the cost matrix ($\xi$) according to a function ($\mathcal{F}$) as follows:
$$
 \mathbf{y}^{(i)} = \mathcal{F}(\xi_p, \mathbf{o}^{(i)}), \quad
: \;  y_{p}^{(i)} \geq y_{j}^{(i)}, \, \forall j \neq p,
 $$
  where, $\mathbf{y}$ denotes the modified  output, $p$ is the desired class and $\mathcal{F} : \mathbb{R}\rightarrow\mathbb{R}$ represents a function whose exact definition depends on the type of loss layer. As an example, for the case of \rev{cost-sensitive} MSE loss, $\mathcal{F}(\xi_p, \mathbf{o}^{(i)}) = \text{sigmoid}(\xi_p \circ \mathbf{o}^{(i)})$, where $\circ$ denotes the hadamard product.  In Sec.~\ref{sec:CSL}, we will discuss in detail the definition of $\mathcal{F}$ for different surrogate losses.
 Note that the \rev{score-level} costs perturb the classifier confidences.
 Such perturbation allows the classifier to give more importance to the less frequent and difficult-to-separate classes.

\subsubsection*{\textbf{Properties of the {Proposed} Cost Matrix $\xi$}}
 Next, we discuss few properties {(lemmas~\ref{lem:three} --~\ref{lem:six})} of the newly introduced cost matrix $\xi$ and its similarities/differences with the traditionally used cost matrix $\xi'$ {(Sec.~\ref{sec:prob_for2})}. The \rev{proofs of} below mentioned properties can be found in Appendix~A (supplementary material):
 \begin{lemma}\label{lem:three}
 The cost matrix $\xi$ for a cost-insensitive loss function is an all-ones matrix, $\mathbf{1}^{p\times p}$, rather than a $\mathbf{1}-\mathbf{I}$ matrix,  
{as in} the case of the traditionally used cost matrix $\xi'$. 
 \end{lemma}
 \begin{lemma}\label{lem:four}
 All costs in $\xi$ are positive, i.e., $\xi \succ 0$. 
 \end{lemma}
 \begin{lemma}
 The cost matrix $\xi$ is defined such that all {of its} elements in are within the range $(0,1]$, i.e., $ \xi_{p,q} \in (0,1]$.
 \end{lemma}
 \begin{lemma}\label{lem:six}
 Offsetting the columns of {the} cost matrix $\xi$ can lead to an equally probable guess point. 
 \end{lemma}
\vspace{.13cm}
{The cost matrix $\xi$ configured according to the properties described above (Lemma~\ref{lem:three} --~\ref{lem:six}) neither excessively increases the CNN outputs activations, nor does it reduce them to zero output values. This
enables a smooth training process allowing the model parameters to be correctly updated. In the following section, we analyse the implications of the newly introduced cost matrix $\xi$ on the loss layer (Fig.~\ref{fig:overview_fig}).}

 \begin{figure}[t]
 \centering
 \includegraphics[width=0.85\columnwidth]{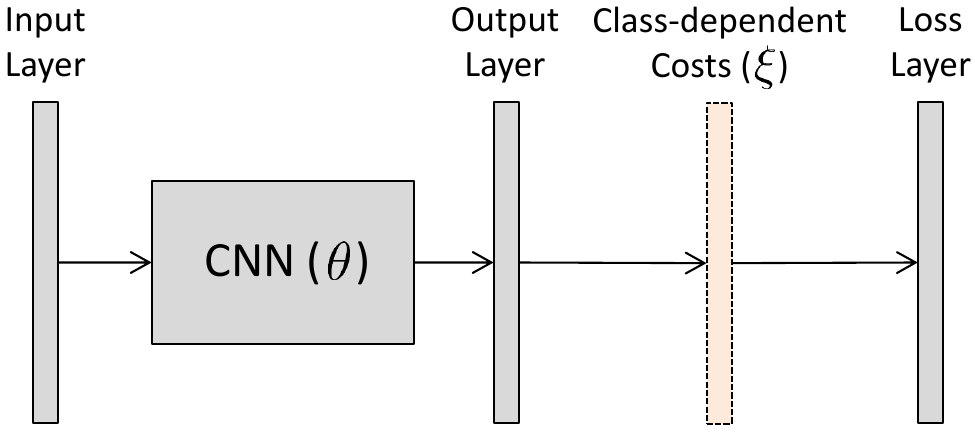}
 \caption{The CNN parameters ($\theta$) and \rev{class-dependent} costs ($\xi$) used during the training process of our deep network. {Details about the CNN architecture and the loss layer are in Sec.~\ref{sec:cnn} and \ref{sec:CSL}, respectively}} 
 \label{fig:overview_fig}\vspace{-0.2cm}
 \end{figure}

\subsection{Cost-Sensitive Surrogate Losses}\label{sec:CSL}
{Our approach addresses} the class imbalance problem during the training of CNNs.
For this purpose, we introduce a \rev{cost-sensitive} error function which can be expressed as the mean loss over the training set:
\begin{align} \label{eq:error_fun}
E(\theta, \xi) = \frac{1}{M}\sum\limits_{i=1}^{M} \ell(\mathbf{d}^{(i)}, \mathbf{y}_{\theta,\xi}^{(i)}),
\end{align}
where, the predicted output ($\mathbf{y}$) of the penultimate layer (before the loss layer) is parameterized by $\theta$ (network weights and biases) and $\xi$ (class sensitive costs), $M$ is the total number of training examples, $\mathbf{d} \in \{0,1\}^{1\times N}$ is the desired output (s.t. $\sum_n d_n \coloneqq 1$) and $N$ denotes the total number of neurons in the output layer.
{For conciseness, we will not explicitly mention the dependence of $\mathbf{y}$ on the parameters ($\theta, \xi$) and {only consider} a single data instance in the discussion below}. 
Note that the error is larger when the model performs poorly on the training set.
The objective of the learning algorithm is to find the optimal parameters $(\theta^*, \xi^*)$ which give the minimum possible cost $E^*$ (Eq.~(\ref{eq:error_fun})). Therefore, the optimization objective is given by:
\begin{align}\label{eq:opt_obj}
(\theta^*, \xi^*) = \underset{\theta, \xi}{\arg\min} \; E(\theta, \xi).
\end{align}

{The loss function $\ell(\cdot)$ in Eq.~(\ref{eq:error_fun}) can be any suitable surrogate loss such as the Mean Square Error (MSE), Support Vector Machine (SVM) hinge loss or a Cross Entropy (CE) loss (also called the  {`soft-max log loss'}).
These popular loss functions are shown along-with other surrogate losses in Fig.~\ref{fig:SurrogateLF}.
The \rev{cost-sensitive} versions of these loss functions are discussed below:}

\vspace{4pt}
\noindent
\textbf{(a) \rev{Cost-Sensitive} MSE loss:} 
{This loss} minimizes the squared error of the predicted output with the desired ground-truth and can be expressed as follows:
\begin{align} \label{eq:MSE_loss} 
\ell(\mathbf{d}^{}, \mathbf{y}_{}^{}) =  \frac{1}{2}\sum\limits_{n}( {d}^{}_{n} - y_n^{}) ^2
\end{align}
where, $y^{}_n$ is related to the output of the previous layer $o^{}_n$ via the logistic function, 
\begin{align}
y^{}_n = \frac{1}{1 + \exp(-{o}_{n}^{} \xi_{p,n})}, \end{align}
where, $\xi$ is the class sensitive penalty which depends on the desired class of a particular training sample, i.e., $p = {\operatorname{argmax}_m\; d_m}$. 
The effect of this cost on the back-propagation algorithm is discussed in Sec.~\ref{subsec:csmseloss}.

\vspace{4pt}
\noindent
\textbf{(b) \rev{Cost-Sensitive} SVM hinge loss:}
{This loss} maximizes the margin between each pair of classes and can be expressed as follows:
\begin{align} \label{eq:svm_loss}
\ell(\mathbf{d}^{}, \mathbf{y}_{}^{}) = - \sum\limits_{n} \max(0, 1 - (2d_n^{} - 1) y_n^{}) ,
\end{align}
where, $y_n$ can be represented in terms of the previous layer output $o^{}_n$ and the cost $\xi$, as follows: 
\begin{align}
y_n = o^{}_n\xi_{p,n}.
\end{align} 
The effect of the introduced cost on the gradient computation is discussed in Sec.~\ref{subsec:cssvmloss}.

\vspace{4pt}
\noindent
\textbf{(c) \rev{Cost-Sensitive} CE loss:}
{This loss} maximizes the closeness of the prediction to the desired output and is given by:
\begin{align}\label{eq:softmaxloss}
\ell(\mathbf{d}^{}, \mathbf{y}_{}^{}) = - \sum\limits_{n} ({d}^{}_{n} \log y^{}_{n}) ,
\end{align}
where $y_n$ incorporates the \rev{class-dependent} cost ($\xi$) and is related to the output $o_n$ via the soft-max function, 
\begin{align}\label{eq:softmax}
y^{}_{n} = \frac{\xi_{p,n}\exp({{o}_{n}^{}})}{\sum\limits_{k} \xi_{p,k} \exp({{o}_{k}^{}})} .
\end{align}
The effect of the modified CE loss on the back-propagation algorithm is discussed in Sec.~\ref{subsec:csceloss}.

\begin{figure}
\centering
\includegraphics[width=0.8\columnwidth]{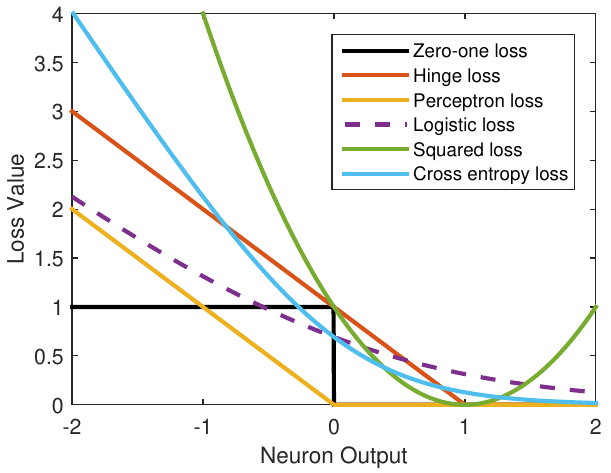}
\caption{{This figure shows the 0-1 loss along-with several other common surrogate loss functions that are used for binary classification.}} 
\label{fig:SurrogateLF}\vspace{-0.2cm}
\end{figure}

\subsubsection*{\textbf{{Classification Feasibility of Cost-Sensitive Losses}}}
Next, we show {(Lemmas~\ref{lem:seven}--\ref{lem:nine})} that the cost-sensitive loss functions remain suitable for classification since they satisfy the following properties: 
\begin{enumerate}
\item Classification Calibration \cite{bartlett2006convexity}
\item Guess Aversion \cite{beijbom2014guess}
\end{enumerate}
Note that a classification calibrated (c-calibrated) loss is useful because the minimization of {the} empirical risk leads to classifiers which have risks {that are} closer to the Bayes-risk.
Similarly, guess aversion {implies} that {the} loss function 
{favours} `correct classification' 
{instead of `arbitrary guesses'}.
Since, CE loss usually performs best among the three loss functions we discussed above \cite{nielsen2014neural,lee2015deeply}, 
{Lemmas~\ref{lem:seven}--\ref{lem:nine} show} that the \rev{cost-sensitive} CE loss is guess aversive and classification calibrated.

\begin{lemma}\label{lem:seven}
For a real valued $\xi$ ($\xi \in \mathbb{R}^{C \times C} \in (0,1]$), given $\mathbf{d}^{(i)}$ and {the} CNN output $\mathbf{o}^{(i)}$, the modified \rev{cost-sensitive} CE loss will be guess-averse iff,
$$
L(\xi, \mathbf{d}^{}, \mathbf{o}^{}) < 
L(\xi, \mathbf{d}^{}, \mathbf{g}) ,
$$
where, $\mathbf{g}$ is the set of all guess points. 
\end{lemma}
\begin{proof}
For real valued CNN activations, the guess point maps to an all zero output:
$$
L(\xi, \mathbf{d}^{}, \mathbf{o}^{}) < 
L(\xi, \mathbf{d}^{}, \mathbf{0}) ,
$$
$$
- \log \left( \frac{\xi_{p,n} \exp (o_n^{})}{\sum_{k} \xi_{p,k} \exp (o_k)}\right) <
-\log \left( \frac{\xi_{p,n} }{\sum_{k} \xi_{p,k}}\right),
$$
which can be satisfied if, 
$$
\frac{\xi_{p,n} \exp (o_n)}{\sum_{k} \xi_{p,k} \exp (o_k^{})} >  \frac{\xi_{p,n} }{\sum_{k} \xi_{p,k}}.
$$
where, $n$ is the true class. Since, $\xi_{p,n} \in (0,1] $ and thus it is $> 0$.
Also, if $n$ is the true class then $o_n > o_k, \, \forall k \neq n$. 
{Therefore, the above relation holds true}.
\end{proof}
\begin{lemma}
The cost matrix has diagonal entries greater than zero, i.e., $\text{diag}(\xi) > 0$. 
\end{lemma}
\begin{proof}
According to Lemma~\ref{lem:zero}, if the CE loss is guess aversive, it must satisfy ,
$$
L(\xi, \mathbf{d}^{}, \mathbf{o}^{}) < 
L(\xi, \mathbf{d}^{}, \mathbf{0}).
$$
We prove the Lemma {by contradiction}.
Let us suppose {that} 
$\xi_{p,n} = 0$, then the above relation does not hold true, since:
$$
\frac{\xi_{p,n} \exp (o_n)}{\sum_{k} \xi_{p,k} \exp (o_k)} =  \frac{\xi_{p,n} }{\sum_{k} \xi_{p,k}} = 0.
$$
{and hence, $\text{diag}(\xi) > 0$.}
\end{proof}

\begin{lemma}\label{lem:nine}
The \rev{cost-sensitive} CE loss function 
$$
\ell(\xi, \mathbf{d}^{}, \mathbf{o}^{}) = - \sum\limits_{n} {d}^{}_{n} \log \left( \frac{\xi_{p,n}\exp({{o}_{n}^{}})}{\sum\limits_{k} \xi_{p,k} \exp({{o}_{k}^{}})} \right) ,
$$
is C-Calibrated.
\end{lemma}
\begin{proof}
{Given an input sample $x^{}$ which} belongs to class $p^*$ (i.e., $d_{p^*}^{} = 1$), then the CE loss can be expressed as:
\begin{align*}
\ell(\xi, \mathbf{d}^{}, \mathbf{o}^{}) = - \log \left( \frac{\xi_{p^*,p^*}\exp({{o}_{p^*}^{}})}{\sum\limits_{k} \xi_{p,k} \exp({{o}_{k}^{}})} \right)
\end{align*}
The classification risk can be expressed in terms of {the} expected value as follows:
\begin{align*}
\mathcal{R}_{\ell}[\mathbf{o}] & = \mathbb{E}_{X,D}[\ell(\xi, \mathbf{d}^{}, \mathbf{o}^{})] \\
& = \sum\limits_{p=1}^{N} P(p|\mathbf{x}^{})\ell(\xi, \mathbf{d}^{}, \mathbf{o}^{})  =  \sum\limits_{p=1}^{N} P(p|\mathbf{x}^{})\ell(\xi, p, \mathbf{o}^{}) \\
& = - \sum\limits_{p=1}^{N}  P(p|\mathbf{x}^{}) \log \left( \frac{\xi_{p,p}\exp({{o}_{p}^{}})}{\sum\limits_{k} \xi_{p,k} \exp({{o}_{k}^{}})} \right)
\end{align*}
Next, we compute the derivative and set it to zero to find the ideal set of CNN outputs `$\mathbf{o}$',
\begin{align*}
& \frac{\partial \mathcal{R}_{\ell}[\mathbf{o}]}{\partial \, {o}_t} 
 = \left. \frac{\partial \mathcal{R}_{\ell}[{o}_p]}{\partial \, {o}_t} \right\rvert_{p\neq t}^{} + \left. \frac{\partial \mathcal{R}_{\ell}[{o}_p]}{\partial \, {o}_t} \right\rvert_{p=t}^{} = 0 \\
& \left. \frac{\partial \mathcal{R}_{\ell}[{o}_p]}{\partial \, {o}_t} \right\rvert_{p= t}^{}  = \frac{\partial}{\partial\,o_t} \left(-\sum\limits_{p=1}^{N} P(p|\mathbf{x}^{}) \log(\xi_{p,p}\exp({{o}_{p}^{}})) \right. \\
& \left. + \sum\limits_{p=1}^{N} P(p|\mathbf{x}^{}) \log(\sum\limits_{k} \xi_{p,k} \exp({{o}_{k}^{}})) \right) \\
& = - P(t|\mathbf{x}^{})  + P(t|\mathbf{x}^{}) \frac{\xi_{t,t} \exp(o_t^{})}{\sum\limits_{k=1}^{N} \xi_{t,k}\exp(o_k^{})}
\end{align*} 
Similarly, 
\begin{align*}
& \left. \frac{\partial \mathcal{R}_{\ell}[{o}_p]}{\partial \, {o}_t} \right\rvert_{p\neq t}^{}  = \sum\limits_{p\neq t}^{}  P(p|\mathbf{x}^{}) \frac{\xi_{p,t} \exp(o_t^{})}{\sum\limits_{k=1}^{N} \xi_{p,k}\exp(o_k^{})} 
\end{align*}

By adding the above two derived expression and setting them to zero, we have :
\begin{align*}
&P(t|\mathbf{x}^{})  = \exp(o_t^{}) \sum\limits_{p=1}^{N}  \frac{P(p|\mathbf{x}^{}) \xi_{p,t} }{\sum\limits_{k=1}^{N} \xi_{p,k}\exp(o_k^{})} \\
& o_t^{} = \log(P(t|\mathbf{x}^{})) - \log \left({ \sum\limits_{p=1}^{N} P(p|\mathbf{x}^{}) \xi_{p,t} }\right) \\
& + \log\left( { \sum\limits_{p=1}^{N}\sum\limits_{k=1}^{N} \xi_{p,k}\exp(o_k^{})} \right)
\end{align*} 
Which shows that there exists an inverse relationship between the optimal CNN output and the Bayes cost of the $t^{th}$ class, and 
{hence,} the cost-sensitive CE loss is classification calibrated. 
\end{proof}
\vspace{0.13cm}

{Under the properties of Lemmas~\ref{lem:seven}--\ref{lem:nine}, the modified loss functions are therefore suitable for classification. Having established the \rev{class-dependent} costs (Sec.~\ref{subsec:propCM}) and their impact on the loss layer (Sec.~\ref{sec:CSL}), we next describe the training algorithm to automatically learn all the parameters of our model ($\theta$ and $\xi$).}

\subsection{{Optimal Parameters Learning}}\label{sec:learning}
When using any of the {previously mentioned} loss functions (Eqs.~(\ref{eq:MSE_loss}-\ref{eq:softmaxloss})), our goal is to jointly learn the hypothesis parameters $\theta$ and the \rev{class-dependent} loss function parameters $\xi$.
For the joint optimization, we alternatively solve for both types of parameters by keeping one fixed and minimizing the cost with respect to the other (Algorithm~\ref{alg:alt_opt}). 
Specifically, for the optimization of $\theta$, we use the stochastic gradient descent (SGD) with the back-propagation of error (Eq.~(\ref{eq:error_fun})). 
{Next, to optimize 
$\xi$}, we again use the gradient descent algorithm to calculate the direction of the step to update the parameters.
{The cost function is also dependent on the class-to-class separability, the current classification errors made by the network with current estimate of parameters and the overall classification error.
The class-to-class (c2c) separability is measured by estimating the spread of the with-in class samples (intraclass) compared to {the} between-class (interclass) ones.
In other words, it measures the relationship between the with-in class sample distances and the size of the separating boundary between {the} different classes. \rev{Note that the proposed cost function can be easily extended to include an externally defined cost matrix for applications where expert opinion is necessary. However, this paper mainly deals with class-imbalance in image classification datasets where externally specified costs are not required. }

\begin{algorithm}[t]
\caption{Iterative optimization for parameters $(\theta, {\xi})$ }
\begin{algorithmic}[1]
\Require Training set ($\mathbf{x}$, $\mathbf{d}$), Validation set ($\mathbf{x}_V$, $\mathbf{d}_V$), \rev{Max. epochs (${M_{ep}}$)}, Learning rate for $\theta$ ($\gamma_{\theta}$), Learning rate for $\xi$ ($\gamma_{\xi}$)
\Ensure Learned parameters ($\theta^{*}$, $\xi^{*}$)
\State Net $\leftarrow$ construct\_CNN$()$
\State $\theta \leftarrow$ initialize\_Net$(\text{Net})$ \Comment{Random initialization}
\State $\mathbf{\xi} \leftarrow \mathbf{1}, \text{val-err} \leftarrow 1$
\For {\rev{$e \in [1,M_{ep}]$}} \Comment{Number of epochs}
\State grad$_\xi \leftarrow $ \rev{compute-grad}$(\mathbf{x}, \mathbf{d}, F(\xi))$ \Comment{Eq.~(\ref{eq:gard_F})}
\State $\xi^{*} \leftarrow \text{update-CostParams}({\xi},  \gamma_{\xi}, \text{grad}_{\xi})$
\State $\xi \leftarrow \xi^{*}$
\For {$b \in [1,B]$} \Comment{Number of batches}
\State ${\text{out}}_b \leftarrow \text{forward-pass}(\mathbf{x}_b, \mathbf{d}_b, \text{Net}, \theta)$
\State grad$_b \leftarrow$ backward-pass$(\text{out}_b, \mathbf{x}_b, \mathbf{d}_b, \text{Net}, \theta, \xi)$
\State $\theta^{*} \leftarrow$ update-NetParams$(\text{Net}, \theta, \gamma_\theta, \text{grad}_b)$
\State $\theta \leftarrow \theta^{*}$
\EndFor
\State val-err$^{*} \leftarrow \text{forward-pass}(\mathbf{x}_V, \mathbf{d}_V, \text{Net}, \theta)$ 
\If {val-err$^{*} >$ val-err}
\State $\gamma_\xi \leftarrow \gamma_\xi * 0.01$ \Comment{Decrease step size}
\State val-err $\leftarrow$ val-err$^{*}$
\EndIf
\EndFor
\State \Return ($\theta^{*}$, $\xi^{*}$)
\end{algorithmic}
\label{alg:alt_opt}
\end{algorithm}

To calculate the c2c separability, we first  compute a suitable distance measure between each point in a class $c_p$ and its nearest neighbour belonging to $c_p$ and the nearest neighbour in class $c_q$. 
Note that these distances are calculated in the feature space where each point is a $4096$ dimensional feature vector (\rev{$f_i: i\in [1,N']$, $N'$ bieng the samples belonging to class $c_p$}) obtained from the {penultimate } CNN layer (just before the output layer). 
Next, we find the average of intraclass distances to interclass distance for each point in a class and compute the ratio of the averages to find the c2c separability index. 
Formally, the class separability between two classes, $p$ and $q$ is defined as:
$$
S(p,q) = \frac{1}{N'}\sum\limits_{i} \frac{dist_{intraNN}(f_i)}{dist_{interNN}(f_i)}
$$ 
To avoid over-fitting and to keep this step computationally feasible, we measure {the} c2c separability  on a small validation set.
Also, the c2c separability was found to correlate well with the confusion matrix at each epoch.
Therefore the measure was calculated after every 10 epochs to minimize the computational  overhead. 
{Note that by simply setting the parameters ($\xi$) based on the percentages  of the classes in the data distribution results in {a} poor performance (Sec.~\ref{subsec:results}). This suggests that the  optimal parameter values for \rev{class-dependent} costs ($\xi^*$) should not be the same as the frequency of the classes in the training data distribution}.
}
The following cost function is used for the gradient computation to update $\xi$:
\begin{align}\label{eq:2ndcost_fun}
F(\xi) =  \parallel T - \xi \parallel_2^2 + E_{val}(\theta,\xi),  
\end{align}
\rev{where $E_{val}$ is the validation error.} The matrix $T$ is defined as follows: 
\begin{align} \label{eq:Texp}
T = H \circ \exp(-\frac{(S-\mu_1)^2}{2\sigma_1^2}) \circ \exp(-\frac{(R-\mu_2)^2}{2\sigma_2^2}) ,
\end{align}
{where, $\mu, \sigma$ denote the parameters which are set using cross validation, \rev{$R$} denotes the current classification errors as a confusion matrix, $S$ denotes the class c2c separability matrix and $H$ is a matrix defined using the histogram vector $\mathbf{h}$ which encodes the distribution of classes in the training set.} The matrix $H$ and vector $\mathbf{h}$ are linked as follows:
\begin{align}
   H(p,q) = \left\{
     \begin{array}{ll}
       \max(h_p, h_q) & :  p\neq q, (p,q) \in \mathbf{c},\\
       h_p                    & : p = q, p \in \mathbf{c} \\
     \end{array}
   \right.
\end{align}
where, $\mathbf{c}$ is the set of all classes in a given dataset. 
The resulting minimization objective to find the optimal $\xi^{*}$  can be expressed as:
\begin{align}\label{eq:cost_fun2}
\xi^{*} = \underset{\xi}{\arg\min} \; F(\xi) .
\end{align}


\begin{figure*}[!htp]
\centering
\includegraphics[width=0.4\textwidth]{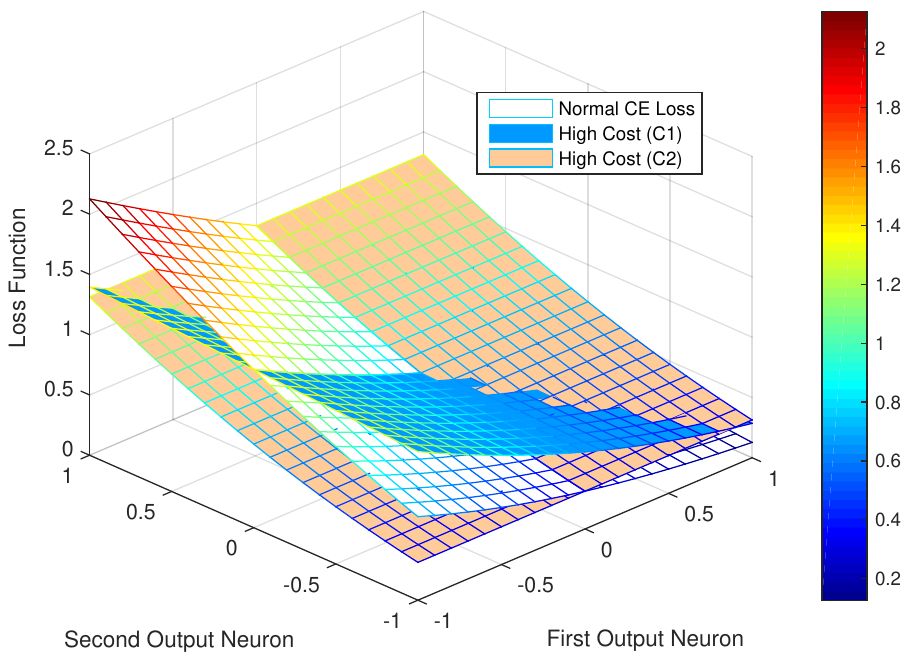}
\qquad
\includegraphics[width=0.4\textwidth]{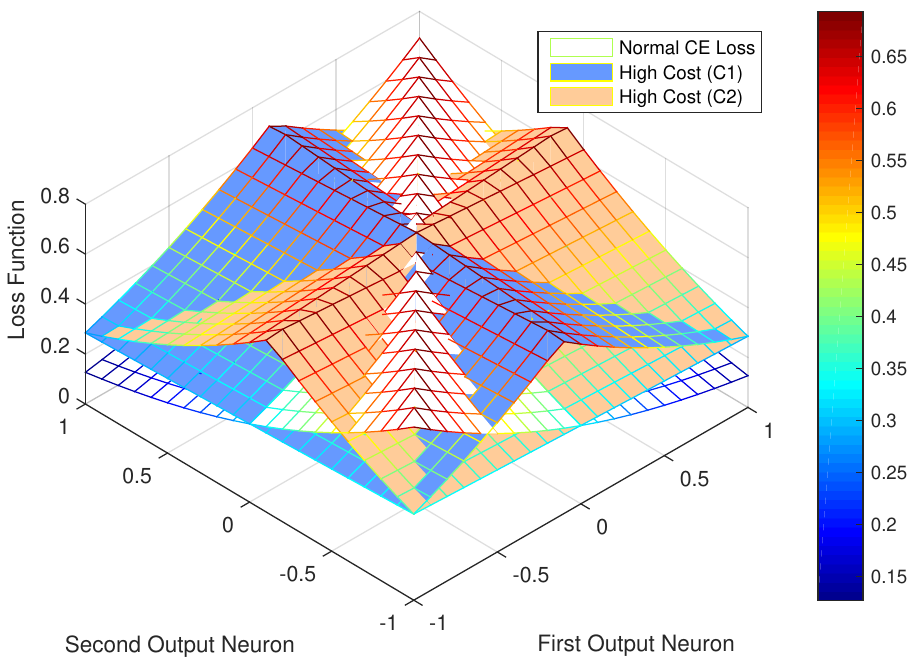}
\vspace{.4cm}
\caption{The CE loss function for the case of binary classification. 
\emph{Left:} loss surface for a single class for different costs (high cost for first class (C1), no cost, high cost for second class (C2)). {\emph{Right:} the minimum loss values for all  possible values of class scores illustrate the obvious  classification boundaries.} The score-level costs reshape the loss surface and the classification boundaries are effectively shifted in favour of the  classes with relatively lower cost. }
\label{fig:loss_func}\vspace{-0.2cm}
\end{figure*}

\noindent
In order to optimize the cost function in Eq.~(\ref{eq:cost_fun2}), we use the gradient descent algorithm which computes the direction of {the} update step, as follows:
\begin{align}\label{eq:gard_F}
 \nabla F(\xi) & =  \nabla  (\mathbf{v}_a - \mathbf{v}_b)(\mathbf{v}_a - \mathbf{v}_b)^T  \notag \\
 &  = (\mathbf{v}_a - \mathbf{v}_b)J_{\mathbf{v}_b}^{T} = -(\mathbf{v}_a - \mathbf{v}_b)\mathbf{1}^{T}.
\end{align}
where, $\mathbf{v}_a = vec(T)$, $\mathbf{v}_b = vec(\xi)$ and  $J$ denotes the \rev{Jacobian matrix}.
Note that {in order} to incorporate the dependence of $F(\xi)$ on the validation error $E_{val}$, we take the update step only if it results in 
{a} decrease 
{in} $E_{val}$ (see Algorithm~\ref{alg:alt_opt}). 

{Since, our approach involves {the use of} modified loss functions during the {CNN} parameter learning process, we will discuss their effect on the back-propagation algorithm in the next section. }

\subsection{{Effect on} Error Back-propagation}\label{sec:BPeff}
{In this section, }we discuss the impact of the modified loss functions on the gradient computation {of}
the back-propagation algorithm. 
\subsubsection{\rev{Cost-Sensitive} MSE}\label{subsec:csmseloss}
During the supervised training, the MSE loss minimizes the mean squared error between the predicted weighted outputs of the model \rev{$\mathbf{y}$}, and the ground-truth labels \rev{$\mathbf{d}$}, across the entire training set  (Eq.~(\ref{eq:MSE_loss})).
The modification of the loss function changes the gradient computed during the back-propagation algorithm. 
Therefore, for the output layer, the mathematical expression of {the} gradient at each neuron is given by:
\begin{align*}
\frac{\partial \ell(\mathbf{d}^{}, \mathbf{y}^{})}{\partial {o}_{n}^{}}  & = - ({d}_{n}^{} - y_n^{}) \frac{\partial y_n^{}}{\partial {o}_{n}^{}} 
\end{align*}
The $y_n^{}$ for the \rev{cost-sensitive} MSE loss can be defined as:
\begin{align*}
y_n^{} = \frac{1}{1 + \exp( - {o}_n^{} \xi_{p,n})}
\end{align*}
The partial derivative can be calculated as follows:
\begin{align*}
\frac{\partial y_n^{}}{\partial {o}_{n}^{}} & = \frac{\xi_{p,n} \exp( - o_n^{} \xi_{p,n})}
{\left(1 + \exp ( - o_n^{} \xi_{p,n})\right)^2}
\\
& = \frac{\xi_{p,n}}{\left(1 + \exp(o_n^{} \xi_{p,n})\right) \left(1 +  \exp(- o_n^{} \xi_{p,n})\right)}
\\
\frac{\partial y_n^{}}{\partial {o}_{n}^{}} & = \xi_{p,n} y_n^{} (1 - y_n^{})
\end{align*}
The derivative of the loss function is therefore given by:
\begin{align}
\frac{\partial \ell(d^{}, {y}^{})}{\partial {o}_{n}^{}}  = - \xi_{p,n}(d_n^{} - y_n^{}) y_n^{} (1 - y_n^{}).
\end{align}

\subsubsection{\rev{Cost-Sensitive} SVM Hinge Loss}\label{subsec:cssvmloss}
For the SVM hinge loss function given in Eq.~(\ref{eq:svm_loss}), the directional derivative can be computed at each neuron as follows:
\begin{align*}
\frac{\partial\ell(\mathbf{d}^{}, {y}_{}^{})}{\partial o_n^{}} = - (2d_n^{} - 1)\frac{\partial y_n^{}}{\partial o_n^{}} \mathbb{I}\{1> y_n^{} (2d_n^{} - 1)\} . 
\end{align*} 
The partial derivative of the output of the softmax w.r.t the output of the penultimate layer is given by:
${\partial y_n^{}}/{\partial o_n^{}} = \xi_{p,n}$.
By combining the above two expressions, the derivative of the loss function can be represented as:
\begin{align}
\frac{\partial\ell(\mathbf{d}^{}, {y}_{}^{})}{\partial o_n^{}} = - (2d_n^{} - 1)\xi_{p,n} \mathbb{I}\{1> y_n^{} (2d_n^{} - 1)\} .
\end{align} 
where, $\mathbb{I}(\cdot)$ denotes an indicator function.

\subsubsection{\rev{Cost-Sensitive} CE loss}\label{subsec:csceloss}
The \rev{cost-sensitive} softmax log loss function is defined in Eq.~(\ref{eq:softmaxloss}).
Next, we show that the introduction of a cost in the CE loss does not change the gradient formulas and the cost is rather incorporated implicitly in the softmax output $y_m^{}$. The effect of costs on the CE loss surface is illustrated in Fig.~\ref{fig:loss_func}.

\begin{prop}
The introduction of a class imbalance cost $\xi_{(\cdot)}$ in the \rev{softmax} loss ($\ell(\cdot)$ in Eq. \ref{eq:softmaxloss}), does not affect the computation of the gradient during the back-propagation process.
\end{prop}
\begin{proof}
We start with the calculation of the partial derivative of the softmax neuron with respect to its input:
\begin{align} 
\frac{\partial y_n^{}}{\partial {o}_{m}^{}} = \frac{\partial}{\partial {o}_{m}^{}} \left( \frac{\xi_{p,n}\exp({{o}_{n}^{}})}{\sum\limits_{k} \xi_{p,k} \exp({{o}_{k}^{}})} \right)
\end{align}
Now, two cases can arise here, either $m=n$ or $m\neq n$. We first solve for the case when $n = m$:
\begingroup\makeatletter\def\f@size{9}\check@mathfonts
$$ LHS = \xi_{p,m} \left( \frac{\exp({o}_{m}^{}) \sum\limits_{k} \xi_{p,k} \exp({{o}_{k}^{}}) - \xi_{p,m} \exp(2{o}_{m}^{}) }{ 
\left(\sum\limits_{k} \xi_{p,k} \exp({{o}_{k}^{}})\right)^2
}\right).$$\endgroup
After simplification we get:
$$ \frac{\partial y_n^{}}{\partial {o}_{m}^{}} = y_m^{} (1 - y_m^{}),  \qquad s.t.: m = n$$
Next, we solve for the case when $n \neq m$:
$$
LHS = - \frac{\xi_{p,n}\xi_{p,n} \exp({o}_{m}^{} + {o}_{n}^{})}
{
\left(\sum\limits_{k} \xi_{p,k} \exp({{o}_{k}^{}})\right)^2
} = -y_m^{} y_n^{}, s.t.: m \neq n.
$$
The loss function can be differentiated as follows:
\begin{align*}
& \frac{\partial \ell(\mathbf{y}^{}, \mathbf{d}^{})}{\partial {o}_{m}^{}}   =  - \sum\limits_{n} {d}^{}_{n} \frac{1}{y_n^{}} \frac{\partial y^{}_{n}}{\partial {o}_{m}^{}}  ,\\
 & = - {d}^{}_{m}(1-y^{}_m) + \sum\limits_{n \neq m}{d}^{}_{n} y^{}_m  =  - {d}^{}_{m} + \sum\limits_{n} {d}^{}_{n} y^{}_m.
\end{align*}
Since, ${d}^{}$ is defined as a probability distribution over all output classes ($ \sum\limits_n {d}^{}_{n} = 1$), therefore:
$$ 
\frac{\partial \ell(\mathbf{y}^{}, \mathbf{d}^{})}{\partial {o}_{m}^{}}  =  - {d}^{}_{m} + y_m^{}.
$$
This result is the same as {in} the case when CE does not contain any \rev{cost-sensitive} parameters. 
Therefore the costs affect the softmax output $y_m^{}$ but the gradient formulas remain unchanged. 
\end{proof}

In our experiments {(Sec.~\ref{sec:exp})}, { we will only report 
performances} with the \rev{cost-sensitive} CE loss  function. 
This is because, {it has been shown that the CE loss 
outperforms} the other two loss functions in most cases \cite{lee2014deeply}.
Moreover, it avoids the learning slowing down problem of the MSE loss \cite{nielsen2014neural}.

\section{Experiments and Results}\label{sec:exp}
The class imbalance problem is present in nearly all real-world object and image  datasets.
This is not because of any flawed data collection, but it is simply due to the natural frequency patterns of different object classes in real life. 
For example, 
 a \emph{bed} 
{appears} in nearly every bedroom scene, but a \emph{baby cot} 
{appears much} less frequently.
{{Consequently, from the 
perspective} of class imbalance, {the} currently available image classification datasets can be divided into three categories:
 \begin{enumerate}
 \item Datasets with {a} significant class imbalance {both in the training 
 and the testing split} (e.g., DIL, MLC),
\item Datasets 
{with unbalanced class distributions 
but with 
experimental protocols that are designed 
in a way that 
an equal number of images from all classes are used during the training process}
(e.g., MIT-67, Caltech-101). The testing images can be equal or unequal for different classes.
 \item Datasets with {an} equal representation of each class in the training and testing splits (e.g., MNIST, CIFAR-100).
 \end{enumerate}
We perform extensive experiments on six challenging image classification datasets (two {from each} 
category) 
(see Sec.~\ref{subsec:datasets}).
For the {case of} imbalanced datasets ($1^{st}$ category), we report results on the standard splits for two experiments. 
{For the two
datasets 
from the} $2^{nd}$ category, we report our performances on the standard splits, deliberately deformed splits and the original data distributions.
For the {two datasets from
the} $3^{rd}$ category, we report results on the standard splits and {on} deliberately imbalanced splits. 
Since, our training procedure requires a small validation set (Algorithm~\ref{alg:alt_opt}), we use $\sim 5\%$ of the training data in each experiment as a held-out validation set.}


\rev{\subsection{Multi-class Performance Metric}\label{subsec:perfmetric}
The main goal of this work is to enhance the overall classification accuracy without compromising the  precision of minority and majority classes. 
Therefore, we report overall classification accuracy results in Tables~\ref{tab:DILexp}-\ref{tab:MIT67exp}, \ref{tab:comp_imbalanced} and \ref{tab:comp_fixedcost} for comparisons with baseline and state-of-the art balanced and unbalanced data classification approaches. We report class recall rates in confusion matrices displayed in Fig.~\ref{fig:CM}. We also show our results in terms of G-mean and F-measure scores on all the six datasets (see Table~\ref{tab:gnf}). Note that the F-measure and G-mean scores are primarily used for binary classification tasks. Here, we extend them to multi-class problem using the approach in \cite{espindola2005extending}, where these scores are calculated for each class in a one-vs-all setting and their weighted average is calculated using the class frequencies.}

\rev{It is also important to note that neural networks give a single classification score and it is therefore not feasible to obtain ROC curves. As a result, we have not included AUC measurements in our experimental results.}

\begin{figure*}[htp]
\includegraphics[width=1\textwidth]{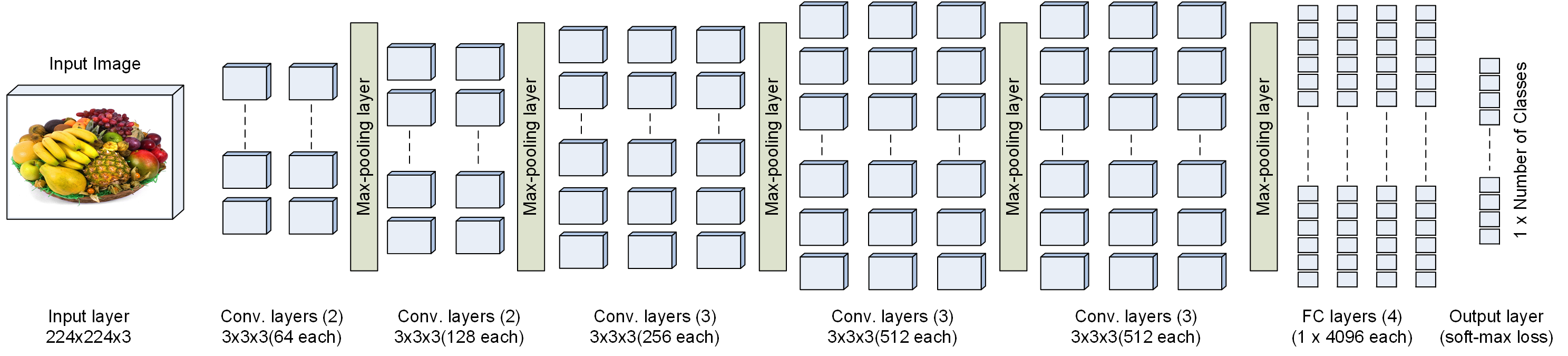}
\caption{The CNN architecture used in this work
{consists} of 18 weight layers. }\vspace{-0.2cm}
\label{fig:cnn_arch}
\end{figure*}

\begin{figure*}[t]
\centering
\begin{subfigure}[t]{0.24\textwidth}
\includegraphics[width=\textwidth]{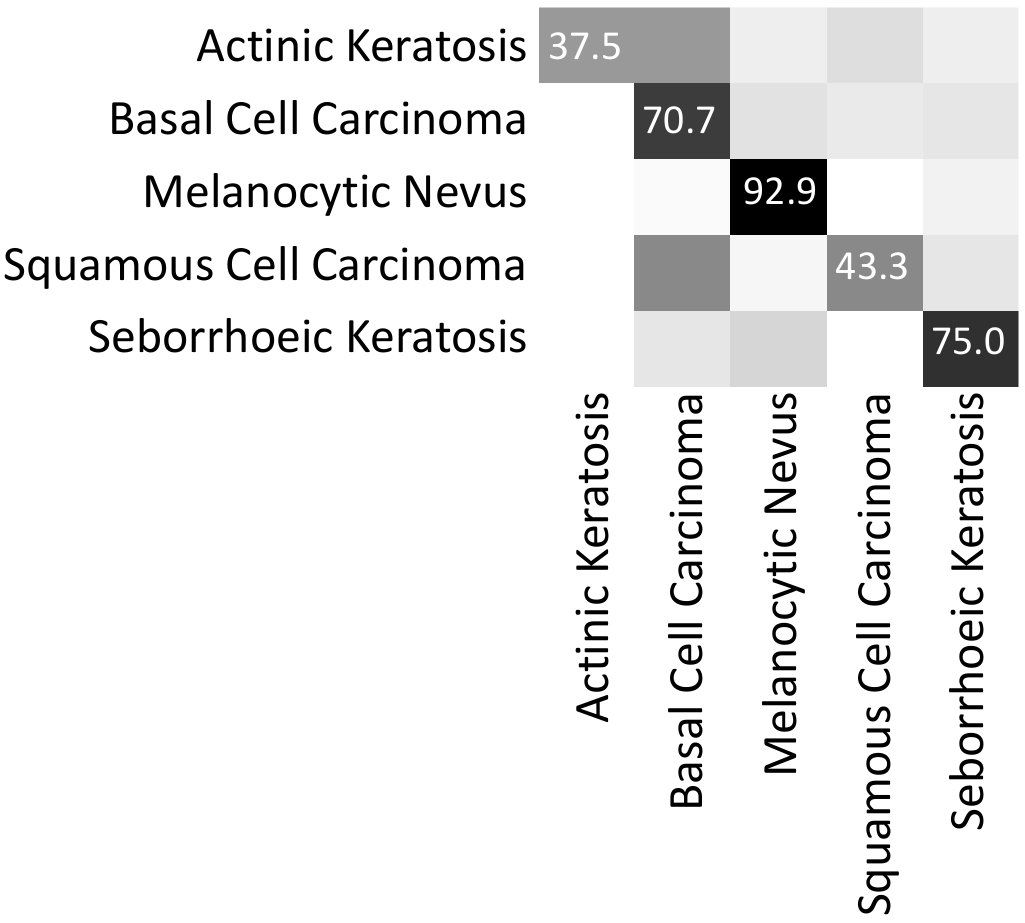}
\caption{DIL (Baseline-CNN)}
\label{fig:DIL_CM_woC}
\end{subfigure}
\hfill
\begin{subfigure}[t]{0.24\textwidth}
\includegraphics[width=\textwidth]{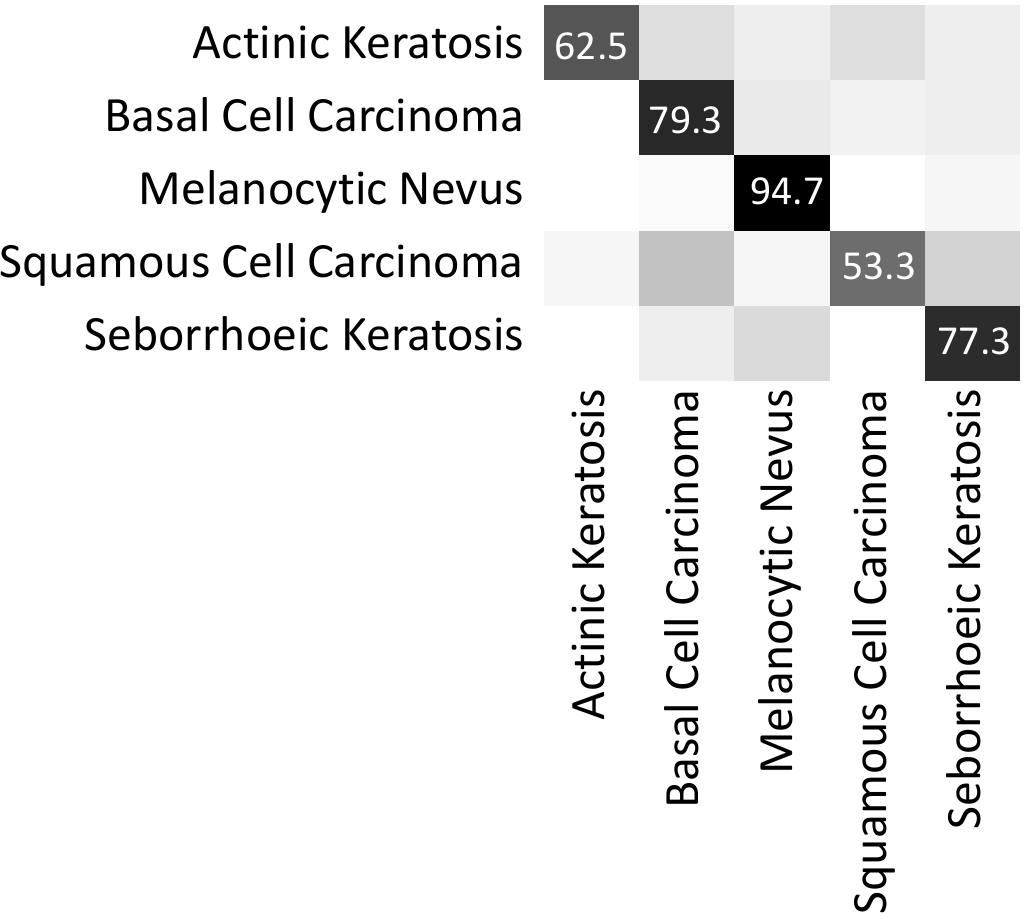}
\caption{DIL (CoSen-CNN)}
\label{fig:DIL_CM_wC}
\end{subfigure}
\hfill
\begin{subfigure}[t]{0.24\textwidth}
\includegraphics[width=\textwidth]{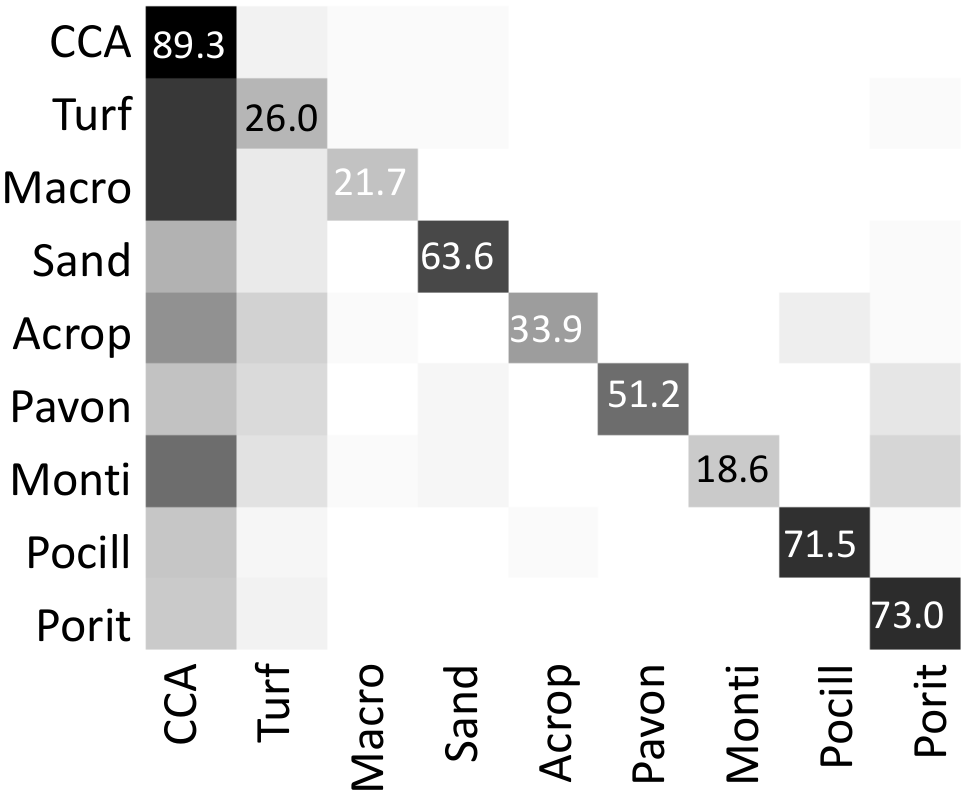}
\caption{MLC (Baseline-CNN)}
\label{fig:MLC_CM_woC}
\end{subfigure}
\hfill
\begin{subfigure}[t]{0.24\textwidth}
\includegraphics[width=\textwidth]{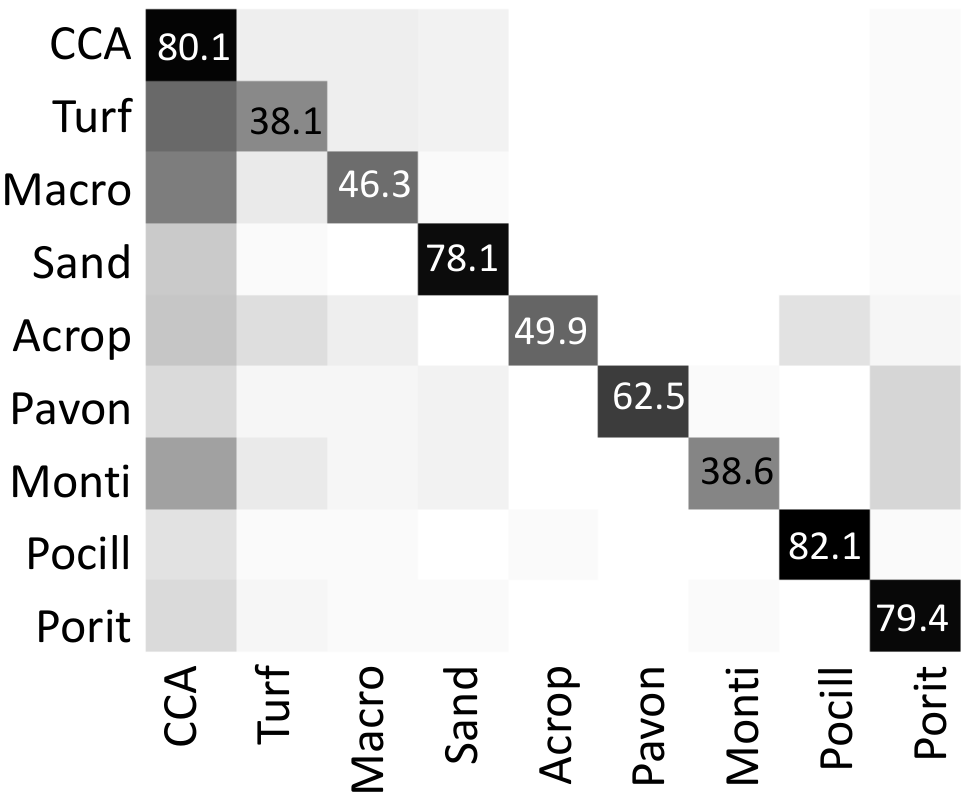}
\caption{MLC (CoSen-CNN)}
\label{fig:MLC_CM_wC}
\end{subfigure}
\vspace{3mm}
\caption{ Confusion Matrices for the 
Baseline and CoSen CNNs on the DIL and MLC datasets. Results are reported for Experiments 1 and 2 for the DIL and MLC datasets, respectively.  }
\label{fig:CM}\vspace{-0.2cm}
\end{figure*}

\subsection{Datasets and Experimental Settings}\label{subsec:datasets}
{\subsubsection{Imbalanced Datasets}
\noindent
\textbf{Melanoma Detection : Edinburgh Dermofit Image Library (DIL)}
consists of 1300 high quality skin lesion images based on diagnosis from dermatologists and dermatopathologists.
There are 10 types of lesions identified in this dataset including melanomas, seborrhoeic keratosis and basal cell carcinomas.
The number of images in each category varies between 24 and 331 (mean 130, median 83).
Similar to \cite{ballerini2013color}, we report results with 3-fold cross validation. 

\vspace{4pt}
\noindent
\textbf{Coral Classification : Moorea Labelled Corals (MLC) } contains 2055 images from three coral reef habitats during 2008-10. Each image is annotated with roughly 200 points belonging to the 9 classes (4 non-corals, 5 corals). Therefore in total, there are nearly 400,000 labelled points.
The class representation varies approximately from 2622 to 196910 (mean 44387, median 30817). 
We perform two of the major standard experiments on this dataset similar to \cite{beijbom2012automated}. 
The first experiment involves training and testing on data from year 2008.
In the second experiment, training is carried out on data from year 2008 and testing on data from year 2009.
}

\subsubsection{Imbalanced Datasets-Balanced Protocols}
\noindent
\textbf{Object Classification: Caltech-101}
contains a total of 9,144 images, divided into 102 categories (101 objects + background). 
The number of images for each category varies between 31 and 800 images (mean: 90, median 59).
{ The dataset is originally imbalanced but the standard protocol which is balanced uses 30 or 15 images for each category during training, and testing is performed on the remaining images (max. 50).}
We perform experiments using the standard $60\%/40\%$ and $30\%/70\%$ train/test splits.

\vspace{4pt}
\noindent
\textbf{Scene Classification: MIT-67} consists of 15,620 images belonging to 67 classes. 
The number of images varies between 101 and 738 (mean: 233, median: 157).
The standard protocol uses a subset of 6700 images (100 per class) for training and evaluation to make the distribution uniform. 
We will, however, evaluate our approach both on the standard split (80 images for training, 20 for testing) and the complete dataset with imbalanced train/test splits of $60\%/40\%$ and $30\%/70\%$.

\subsubsection{Balanced Datasets-Balanced Protocols}
\noindent
\textbf{Handwritten Digit Classification: MNIST} consists of 70,000 images of digits (0-9). 
Out of the total, 60,000 images are used for training ($\sim$600/class) and the remaining 10,000 for testing ($\sim$100/class). 
We evaluate our approach on the standard split as well as the deliberately imbalanced splits. 
To imbalance the training distribution, we reduce the representation of even and odd digit classes to only $25\%$ and $10\%$ of images, respectively.

\vspace{4pt}
\noindent
\textbf{Image Classification: CIFAR-100}
contains 60,000 images belonging to 100 classes (600 images/class). 
The standard train/test split for each class is 500/100 images.
We evaluate our approach on the standard split as well as on artificially imbalanced splits. 
To imbalance the training distribution, we reduce the representation of even-numbered and odd-numbered classes to only $25\%$ and $10\%$ of images, respectively.

\begin{table}
\centering
\scalebox{1}{
\begin{tabular}{@{}lcc@{}}
  \toprule[0.4mm]
  \textbf{Methods}  &  \multicolumn{2}{c}{\textbf{Performances}}  \\
  (using stand. split) & Exp. 1 (5-classes) & Exp. 2 (10-classes) \\
  \midrule
  Hierarchical-KNN \cite{ballerini2013color} &{74.3 $\pm$ 2.5\%}  & 68.8 $\pm$ 2.0\%\\
  Hierarchical-Bayes \cite{ballerini2012non} & {69.6 $\pm$ 0.4\%} & 63.1 $\pm$ 0.6\% \\
  Flat-KNN \cite{ballerini2013color} & {69.8 $\pm$ 1.6\%} & 64.0 $\pm$ 1.3\%\\
  \midrule
  Baseline CNN & {75.2 $\pm$ 2.7\%} & {69.5 $\pm$ 2.3\%} \\ 
  CoSen CNN  & \textbf{80.2 $\pm$ 2.5\%} & \textbf{72.6 $\pm$ 1.6\%}  \\
  \bottomrule[0.4mm]
  \vspace{0.1mm}
\end{tabular}
}
\caption{Evaluation on DIL Database.}\vspace{-0.2cm}
\label{tab:DILexp}
\end{table}

\subsection{Convolutional Neural Network}\label{sec:cnn}
We use a deep CNN to learn robust feature representations for the task of image classification. 
The network architecture consists of a total of 18 weight layers (see Fig.~\ref{fig:cnn_arch} for details).
Our architecture is similar to the state-of-the-art CNN (configuration D) proposed in \cite{simonyan2014very}, except that our architecture has two extra fully connected layers before the output layer and the proposed loss layer is \rev{cost-sensitive}.
Since there are a huge number of parameters  ($\sim$139 million) in the network, its not possible to learn all of them from scratch using a relatively smaller number of images.
We, therefore, initialize the first 16 layers of our model with the pre-trained model of \cite{simonyan2014very} and set random weights for the last two fully connected layers.
We then train the full network with a relatively higher learning rate to allow a change in the network parameters.
Note that the \rev{cost-sensitive} (CoSen) CNN is trained with the modified cost functions introduced in Eqs.~(\ref{eq:MSE_loss}-\ref{eq:softmax}).
The CNN trained without \rev{cost-sensitive} loss layer will be used as the baseline CNN in our experiments. 
\rev{Note that the baseline CNN architecture is exactly the same as the CoSen CNN, except that the final layer is a simple CE loss layer.}

\begin{table}
\centering
\scalebox{1}{
\begin{tabular}{@{}lcc@{}}
  \toprule[0.4mm]
  \textbf{Methods}  &  \multicolumn{2}{c}{\textbf{Performances}}  \\
  (using stand. split) & Exp.~1 (2008) & Exp.~2 (2008-2009) \\
  \midrule
   MTM-CCS (LAB) \cite{beijbom2012automated} &{74.3\%}  & 67.3\%\\
   MTM-CCS (RGB) \cite{beijbom2012automated} &{72.5\%}  & 66.0\%\\
  \midrule
  Baseline CNN & {72.9\%} & {66.1\%} \\ 
  CoSen CNN  & \textbf{75.2\%} & \textbf{68.6\%}  \\
  \bottomrule[0.4mm]
  \vspace{0.1mm}
\end{tabular}
}
\caption{ Evaluation on MLC Database.}\vspace{-0.4cm}
\label{tab:MLCexp}
\end{table}

\subsection{Results and Comparisons}\label{subsec:results}
 For the two imbalanced datasets with imbalanced protocols, we summarize our experimental results and comparisons in Tables~\ref{tab:DILexp},~\ref{tab:MLCexp}.
For each of the two datasets, we perform two standard experiments following the works of Beijbom \etal \cite{beijbom2012automated} and Ballerini \etal \cite{ballerini2012non}.
In the first experiment on the DIL dataset, we perform 3-fold cross validation on the 5 classes (namely Actinic Keratosis, Basal Cell Carcinoma, Melanocytic Nevus, Squamous Cell Carcinoma and Seborrhoeic Keratosis) comprising {of} a total of 960 images. 
In the second experiment, we perform 3-fold cross validation on all of the 10 classes in the DIL dataset. 
We 
{achieved} a performance boost of $\sim 5.0\%$ and $\sim 3.1\%$ over the baseline CNN in the first and second experiments respectively (Table~\ref{tab:DILexp}).

\begin{table}
\centering
\scalebox{1}{
\begin{tabular}{@{}lcc@{}}
  \toprule[0.4mm]
  \textbf{Methods} (using stand. split) &  \multicolumn{2}{c}{\textbf{Performances}}  \\
  \midrule
  Deeply Supervised Nets \cite{lee2015deeply} & \multicolumn{2}{c}{{99.6\%}}\\
  \rev{Generalized Pooling Func. \cite{lee2016generalizing}} & \multicolumn{2}{c}{99.7\%} \\
  \rev{Maxout NIN \cite{chang2015batch}} & \multicolumn{2}{c}{\textbf{99.8\%}} \\
  \midrule
  \textbf{Our approach} ($\downarrow$) &  Baseline CNN & CoSen CNN \\
  \midrule
  Stand. split ($\sim$600 trn, $\sim$100 tst) & {99.3\%} & {99.3\%} \\ 
  \hdashline 
  Low rep. (10\%) of  odd digits &  97.6\%  & \textbf{98.6\%}\\
  Low rep. (10\%) of  even digits &  97.1\%  & \textbf{98.4\%}\\
  \hdashline 
  Low rep. (25\%) of  odd digits &   98.1\% & \textbf{98.9\%} \\
  Low rep. (25\%) of  even digits &  97.8\% & \textbf{98.5\%}\\
  \bottomrule[0.4mm]
  \vspace{0.1mm}
\end{tabular}
}
\caption{Evaluation on MNIST Database.}\vspace{-0.2cm}
\label{tab:MNISTexp}
\end{table}

\begin{table}
\centering
\scalebox{1}{
\begin{tabular}{@{}lcc@{}}
  \toprule[.4mm]
  \textbf{Methods} (using stand. split) &  \multicolumn{2}{c}{\textbf{Performances}}  \\
  \midrule
Network in Network \cite{lin2014network} &  \multicolumn{2}{c}{64.3\%} \\
Probablistic Maxout Network \cite{springenberg2014improving} & \multicolumn{2}{c}{61.9\%} \\
Representation Learning \cite{lin2014stable} & \multicolumn{2}{c}{60.8\%} \\
Deeply Supervised Nets \cite{lee2015deeply} & \multicolumn{2}{c}{{65.4\%}} \\
\rev{Generalized Pooling Func. \cite{lee2016generalizing}} & \multicolumn{2}{c}{67.6\%} \\
  \rev{Maxout NIN \cite{chang2015batch}} & \multicolumn{2}{c}{\textbf{71.1\%}} \\
 \midrule
 \textbf{Our approach} ($\downarrow$) & Baseline CNN & CoSen CNN \\
 \midrule
 Stand. split (500 trn, 100 tst) & {65.2\%} & {65.2\%} \\ 
  \hdashline 
  Low rep. (10\%) of  odd digits &  55.0\%  & \textbf{60.1\%}\\
  Low rep. (10\%) of  even digits &  53.8\%  & \textbf{59.8\%}\\
  \hdashline 
  Low rep. (25\%) of  odd digits &  57.7\% & \textbf{61.5\%}\\
  Low rep. (25\%) of  even digits & 57.4\%  & \textbf{61.6\%}\\
  \bottomrule[.4mm]
  \vspace{0.1mm}
\end{tabular}
}
\caption{Evaluation on CIFAR-100 Database.}\vspace{-0.4cm}
\label{tab:CIFARexp}
\end{table}

For the MLC dataset, in the first experiment we train on two-thirds of the data from 2008 and test on the remaining one third.
In the second experiment, data from year 2008 is used for training and tests are performed on data from year 2009. 
Note that in contrast to the `multiple texton maps' (MTM) \cite{beijbom2012automated} approach which extracts features from multiple scales, we only extract features from the $224\times 224$ dimensional patches.
{While we can achieve a larger gain by using multiple scales with our approach,}   
we kept the setting similar to the one used {with the other datasets for consistency}. 
For similar reasons, we used {the} RGB color space instead of LAB, which was shown to perform better on the MLC dataset \cite{beijbom2012automated}.
Compared to {the} baseline CNN, we 
{achieved} a gain of $2.3\%$ and $2.5\%$ on  the first and 
second experiments respectively. 
{Although the gains in the overall accuracy may seem modest, it should be noted that the boost in the average class accuracy is more pronounced}. 
{For example, 
the confusion matrices for DIL and MLC datasets in Fig.~\ref{fig:CM} (corresponding to Exp.~1 and Exp.~2 respectively), show an improvement of $9.5\%$ and $11.8\%$ in the average} class accuracy.
The confusion matrices in Figs.~\ref{fig:DIL_CM_woC},~\ref{fig:DIL_CM_wC},~\ref{fig:MLC_CM_woC} and \ref{fig:MLC_CM_wC} also show  a very significant boost in performance for the least frequent classes e.g., Turf, Macro, Monti, AK and SCC.

\setlength\dashlinedash{3.6pt}
\setlength\dashlinegap{3.5pt}
\setlength\arrayrulewidth{0.3pt}
\begin{table}
\centering
\scalebox{1}{
\begin{tabular}{@{}lcc@{}}
  \toprule[.4mm]
   \textbf{Methods} (using stand. split) &  \multicolumn{2}{c}{\textbf{Performances}}  \\
   \cmidrule{2-3}
   &  15 trn. samples & 30 trn. samples \\
  \midrule
    Multiple Kernels \cite{vedaldi2009multiple} & 71.1 $\pm$ 0.6 & 78.2 $\pm$ 0.4 \\
   LLC$^\dagger$  \cite{wang2010locality} & $-$ & 76.9 $\pm$ 0.4\\
   Imp. Fisher Kernel$^\dagger$ \cite{perronnin2010improving} & $-$ & 77.8 $\pm$ 0.6\\
  SPM-SC \cite{yang2009linear} & 73.2 & 84.3 \\
  DeCAF \cite{donahue2014decaf} & $-$ & 86.9 $\pm$ 0.7\\
   Zeiler \& Fergus \cite{zeiler2014visualizing} & 83.8 $\pm$ 0.5 & 86.5 $\pm$ 0.5 \\
   Chatfield \etal \cite{chatfield2014return} & $-$ & 88.3 $\pm$ 0.6\\
  SPP-net \cite{he2014spatial} & $-$ & {\textbf{91.4} $\pm$ \textbf{0.7}} \\
  \midrule
  \textbf{Our approach} ($\downarrow$) & Baseline CNN & CoSen CNN \\
  \midrule
  Stand. split (15 trn. samples) & \textbf{87.1\%} &{\textbf{87.1\%}} \\
  Stand. split (30 trn. samples) & {90.8\%} &{{90.8\%}} \\
  \midrule
  Org. data distribution & \multirow{2}{*}{88.1\% } & \multirow{2}{*}{\textbf{89.3\%}} \\ 
  (60\%/40\% split) & &  \\
    \hdashline
  Low rep. (10\%) of odd classes  & 77.4\% & \textbf{83.2\%} \\
  Low rep. (10\%) of even classes & 76.1\% & \textbf{82.3\%} \\
  \midrule 
  Org. data distribution & \multirow{2}{*}{85.5\%} & \multirow{2}{*}{\textbf{87.9\%}} \\ 
  (30\%/70\% split) & & \\
    \hdashline
  Low rep. (10\%) of odd classes & 74.6\% & \textbf{80.4\%} \\
  Low rep. (10\%) of even classes & 75.2\% & \textbf{80.9\%} \\
  \bottomrule[0.4mm]
  \vspace{0.1mm}
\end{tabular}
}
\caption{Evaluation on Caltech-101 Database ($\dagger$ figures reported in \cite{chatfield2011devil}).}\vspace{-0.2cm}
\label{tab:Caltechexp}
\end{table}

\begin{table}
\centering
\scalebox{1}{
\begin{tabular}{@{}lcc@{}}
  \toprule[.4mm]
  \textbf{Methods} (using stand. split) &  \multicolumn{2}{c}{\textbf{Performances}}  \\
  \midrule
  Spatial Pooling Regions  \cite{linlearning14} & \multicolumn{2}{c}{$50.1\%$} \\
 VC + VQ  \cite{li2013harvesting} & \multicolumn{2}{c}{$52.3\%$} \\
 CNN-SVM  \cite{razavian2014cnn} & \multicolumn{2}{c}{$58.4\%$} \\
Improved Fisher Vectors \cite{juneja2013blocks}  & \multicolumn{2}{c}{$60.8\%$} \\
 Mid Level Representation  \cite{doersch2013mid} & \multicolumn{2}{c}{$64.0\%$}\\
Multiscale Orderless Pooling \cite{GongMOP14}  & \multicolumn{2}{c}{$68.9\%$} \\
 \midrule
 \textbf{Our approach} ($\downarrow$) & Baseline CNN & CoSen CNN \\
  \midrule
  Stand. split (80 trn, 20 tst) & \textbf{70.9\%} &\textbf{70.9\%} \\
    \midrule
  Org. data distribution & \multirow{2}{*}{70.7\%} & \multirow{2}{*}{\textbf{73.2\%}} \\ 
  (60\%/40\% split) & & \\
  \hdashline 
  Low rep. (10\%) of odd classes & 50.4\%  & \textbf{56.9\%}   \\
  Low rep. (10\%) of even classes & 50.1\%    & \textbf{56.4\%} \\
   \midrule
  Org. data distribution & \multirow{2}{*}{61.9\%} & \multirow{2}{*}{\textbf{66.2\%}} \\ 
  (30\%/70\% split) &   & \\
  \hdashline 
  Low rep. (10\%) of odd classes  & 38.7\% & \textbf{44.7\%} \\
  Low rep. (10\%) of even classes & 37.2\% & \textbf{43.4\%} \\
  \bottomrule[.4mm]
  \vspace{0.1mm}
\end{tabular}
}
\caption{Evaluation on MIT-67 Database.}\vspace{-0.4cm}
\label{tab:MIT67exp}
\end{table}


\begin{table}
\setlength{\tabcolsep}{4pt}
\centering
\begin{tabular}{@{}l@{\;}c@{\;\;}c@{\;\;}c@{\;\;}c@{}}
\toprule[.4mm]
\textbf{Dataset} & \multicolumn{2}{c}{\bf F-measure} & \multicolumn{2}{c}{\bf G-mean} \\
\cmidrule{2-5}
 & Baseline CNN & CoSen CNN & Baseline CNN & CoSen CNN \\
\midrule
MNIST & 0.488 & 0.493 & 0.987 & 0.992 \\
CIFAR-100 & 0.283 & 0.307 & 0.736 & 0.766 \\
Caltech-101 & 0.389 & 0.416 & 0.873 & 0.905 \\
MIT-67 & 0.266 & 0.302 & 0.725 & 0.772 \\
DIL & 0.343 & 0.358 & 0.789 & 0.813 \\
MLC & 0.314 & 0.338 & 0.635 & 0.723 \\
\bottomrule[.4mm]
\vspace{0.1mm}
\end{tabular}
\captionof{table}{\rev{The table shows the F-measure and G-mean scores for the baseline and cost-sensitive CNNs. The experimental protocols used for each dataset are shown in Fig.~\ref{fig:exp_set}. CosSen CNN consistently outperforms the Baseline CNN on all datasets.} }
\label{tab:gnf}
\end{table}

\begin{table*}
\setlength{\tabcolsep}{4pt}
\centering
\scalebox{1}{
\begin{tabular}{@{}lc@{}cccccccc@{}}
\toprule[.4mm]
\textbf{Datasets} &  &  \multicolumn{8}{c}{\textbf{Performances}} \\
\cmidrule{3-10}
(Imbalaned       & Experimental & Over-sampling  & Under-sampling & Hybrid-sampling & CoSen SVM & CoSen RF & \rev{SOSR} &  Baseline  & CoSen \\
 protocols) & Setting & (SMOTE \cite{chawla2002smote}) & (RUS \cite{mani2003knn}) & (SMOTE-RSB$^{*}$\cite{ramentol2012smote}) & (WSVM \cite{tang2009svms}) &(WRF \cite{chen2004using}) &  \rev{CNN \cite{chung2015cost}} & CNN & CNN \\
\midrule
MNIST & 10\% of odd classes  & 94.5\% & 92.1\% & 96.0\% & 96.8\%  & 96.3\% & \rev{97.8\%} & 97.6\% & \textbf{98.6\%} \\
CIFAR-100 & 10\% of odd classes & 32.2\% & 28.8\% & 37.5\% & 39.9\% & 39.0\% & \rev{55.8\%} & 55.0\% & \textbf{60.1\%} \\
Caltech-101 & 60\% trn, 10\% of odd cl. & 67.7\% & 61.4\% & 68.2\% & 70.1\% & 68.7\% & \rev{77.4\%} & 77.4\% & \textbf{83.2\%} \\
MIT-67 & 60\% trn, 10\% of odd cl. & 33.9\% & 28.4\% & 34.0\% & 35.5\% & 35.2\% & \rev{49.8\%} & 50.4\% & \textbf{56.9\%} \\
{DIL} & {stand. split (Exp. 2)} & 50.3\% & 46.7\% & 52.6\% & 55.3\% & 54.7\% & \rev{68.9\%} & 69.5\% & \textbf{72.6\%} \\
{MLC} & {stand. split (Exp. 2)} & 38.9\% & 31.4\% & 43.0\% & 47.7\%& 46.5\% & \rev{65.7\%} & 66.1\% & \textbf{68.6\%}\\
\bottomrule[.4mm]
\vspace{0.1mm}
\end{tabular}
}
\caption{Comparisons of our approach with the \rev{state-of-the-art} class-imbalance approaches. The experimental protocols used for each dataset are shown in Fig.~\ref{fig:exp_set}. With highly imbalanced training sets, our approach significantly out-performs other data sampling and \rev{cost-sensitive} classifiers on all four classification datasets. }
\label{tab:comp_imbalanced}\vspace{-0.2cm}
\end{table*}

\begin{SCfigure*}
\centering
\begin{minipage}[t]{0.20\textwidth}
\centering
\caption{ The imbalanced training set distributions used for the comparisons reported in Tables~\ref{tab:gnf}-\ref{tab:comp_fixedcost}. Note that for the DIL and the MLC datasets, these distributions are {the} same as the standard protocols. For the MLC dataset, only the training set distribution for the first experiment is shown here which is very similar to the training set distribution of {the} second experiment (\emph{best viewed when enlarged}). }
\label{fig:exp_set}
\end{minipage}

\begin{minipage}[t]{0.80\textwidth}
\centering
{\begin{subfigure}[t]{0.31\textwidth}
\includegraphics[width=\columnwidth]{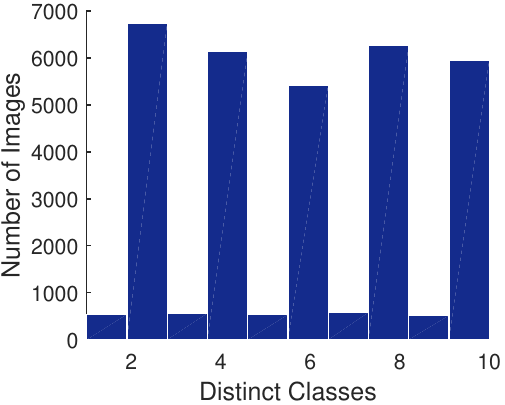}
\caption{MNIST Training Set}
\label{fig:mnist_trn}
\end{subfigure}
\hfill
\begin{subfigure}[t]{0.31\textwidth}
\includegraphics[width=\columnwidth]{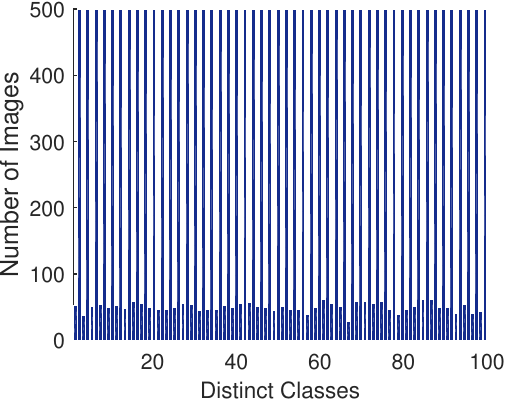}
\caption{CIFAR-100 Training Set}
\label{cifar100_trn}
\end{subfigure}
\hfill
\begin{subfigure}[t]{0.31\textwidth}
\includegraphics[width=\textwidth]{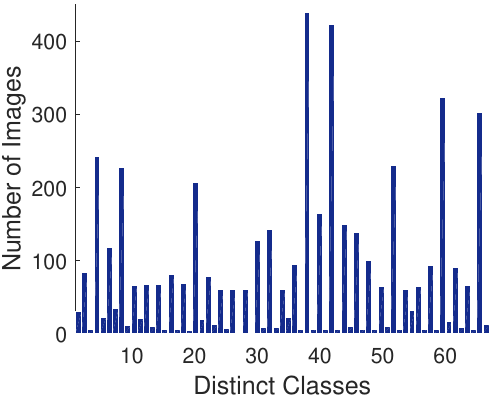}
\caption{MIT-67 Training Set}
\label{fig:mit67_trn}
\end{subfigure}
\newline
\begin{subfigure}[t]{0.31\textwidth}
\includegraphics[width=\textwidth]{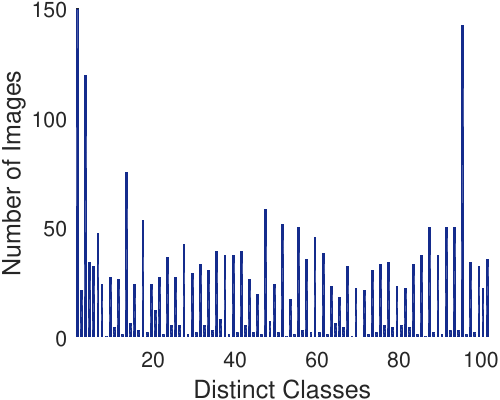}
\caption{Caltech-101 Training Set}
\label{fig:caltech101_trn}
\end{subfigure}
\hfill
\begin{subfigure}[t]{0.28\textwidth}
\includegraphics[width=\textwidth]{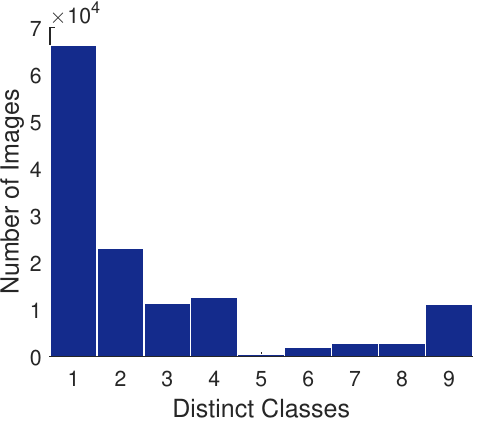}
\caption{MLC Training Set}
\label{fig:mlc_trn}
\end{subfigure}
\hfill
\begin{subfigure}[t]{0.31\textwidth}
\includegraphics[width=\textwidth]{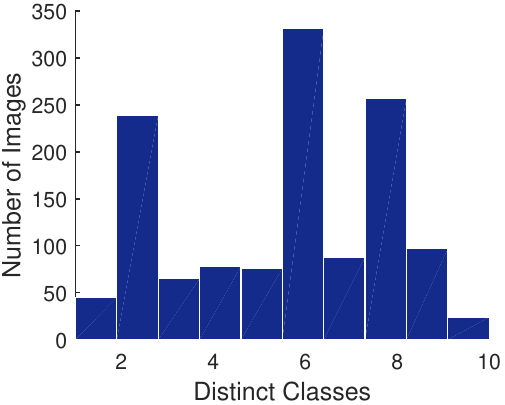}
\caption{DIL Training Set}
\label{fig:dil_trn}
\end{subfigure}}
\end{minipage}

\vspace{-0.4cm}
\label{fig:exp_set}
\end{SCfigure*}

{Our results for the two balanced datasets, MNIST and CIFAR-100, are reported in Tables~\ref{tab:MNISTexp},~\ref{tab:CIFARexp} on} the standard splits along-with the deliberately imbalanced splits.
%
%
{To imbalance the training distributions, we used the available/normal training data for the even classes and only 25\% and 10\% of data for the odd classes.} 
Similarly, we experimented by keeping the normal representation of {the} odd classes  and reducing the representation of {the} even classes to only 25\% and 10\%.
{Our results show that the performance of our approach is equal to the performance of the baseline method when the distribution is balanced, but when the imbalance ratios increase, our approach produces significant improvements over the baseline CNN (which is trained without using the \rev{cost-sensitive} loss layer)}. 
%
We also compare with the \rev{state-of-the-art} techniques which report results on the standard split\footnote{Note that the standard split on {the} Caltech-101 and MIT-67 is different from the original data distribution (see Sec.~\ref{subsec:datasets} for details). } and demonstrate that our performances {are better or comparable}.
Note that for the MNIST digit dataset, nearly all the top performing approaches use distortions (affine and/or elastic) and {data augmentation to achieve} a significant boost in performance.
In contrast, our baseline and \rev{cost-sensitive} CNNs do not use any form of distortions/augmentation during the training and testing procedures on MNIST.

We also experiment on the two popular classification datasets which are originally imbalanced, 
{and for which} the standard protocols use an equal number of images for all training classes. For example, 30 or 15 images are used for the case of Clatech-101 while 80 images per category are used in MIT-67 for training.
We report our results on the standard splits (Tables~\ref{tab:Caltechexp},~\ref{tab:MIT67exp}), to compare with the \rev{state-of-the-art} approaches, and show that our results are superior to the \rev{state-of-the-art} on MIT-67 and competitive on the Caltech-101 dataset. 
Note that the \rev{best-performing} SPP-net \cite{he2014spatial} uses multiple sizes of Caltech-101 images during training.
In contrast, we only use a single {consistent} size during training and testing.
We also experiment with the original imbalanced data distributions to train the CNN with the modified loss function. 
For the original data distributions, we use both 60\%/40\% and 30\%/70\% train/test splits  to show our performances with a variety of train/test distributions. 
Moreover, with these imbalanced splits, we further decrease the data of odd and even classes to just 10\% respectively, and observe a better relative performance of our proposed approach compared to the baseline method.

\rev{We report F-measure and G-mean scores on all the six datasets in Table~\ref{tab:gnf}. The metric calculation details are provided in Sec.~\ref{subsec:perfmetric}. The most unbalanced splits (Fig.~\ref{fig:exp_set}) are used for each dataset to clearly demonstrate the benefit of class-specific costs. We note that the cost-sensitive CNN model clearly out-performs the baseline model for all experiments.}

%
%
%
%
%
%



The comparisons with the best approaches for class-imbalance learning are shown in Table \ref{tab:comp_imbalanced}.
Note that we used a high degree of imbalance for the case of all six datasets to clearly show the impact of the class imbalance {problem on the performance of the different} approaches (Fig.\ref{fig:exp_set}).
For {fairness and conclusive} comparisons, our experimental procedure was kept as close as possible to the  proposed CoSen CNN.
For example, for the case of CoSen Support Vector Machine (SVM) and Random Forest (RF) classifiers, we used the $4096$ dimensional features extracted from the pre-trained deep CNN (D) \cite{simonyan2014very}. 
Similarly, for the cases of over and under-sampling, we used the same $4096$ dimensional features, {which have 
shown} to perform well on other classification datasets.
\rev{A two-layered neural network was used for classification with these sampling procedures.  }
We {also} report comparisons with all types of data sampling techniques i.e., over-sampling (SMOTE \cite{chawla2002smote}), under-sampling (Random Under Sampling - RUS \cite{mani2003knn}) and hybrid sampling (SMOTE-RSB$^{*}$ \cite{ramentol2012smote}). 
Note that despite the simplicity of the approaches  {in} \cite{chawla2002smote, mani2003knn}, they have been shown to perform very well on imbalanced datasets in data mining \cite{garcia2007class,he2009learning}. 
We also compare with the \rev{cost-sensitive} versions of popular classifiers (weighted SVM \cite{tang2009svms} and weighted RF \cite{chen2004using}).
For the case of weighted SVM, we used the standard implementation of LIBSVM \cite{chang2011libsvm} and set the \rev{class-dependent} costs based on the proportion of each class in the training set. \rev{Finally, we experiment with a recent cost-sensitive deep learning based technique of Chung \etal \cite{chung2015cost}. Unlike our approach, \cite{chung2015cost} does not automatically learn class-specific costs. To have a fair comparison, we incorporate their proposed smooth one-sided regression (SOSR) loss as the last layer of the baseline CNN model in our experiments. Similar to \cite{chung2015cost}, we use the approach proposed in \cite{abe2004iterative} to generate fixed cost matrices.}
Our proposed approach demonstrates a significant improvement over all of the \rev{cost-sensitive} class imbalance methods.

\begin{table}
\setlength{\tabcolsep}{4pt}
\centering
\scalebox{0.95}{
\begin{tabular}{@{}l@{\;}c@{\;\;}c@{\;\;}c@{\;}c@{}}
\toprule[.4mm]
\textbf{Datasets} & \multicolumn{4}{c}{\textbf{Performances}} \\
\cmidrule{2-5}
(Imbalaned      & CoSen-CNN  & CoSen-CNN & CoSen-CNN & CoSen-CNN \\
 protocols) & Fixed Cost (H) &  Fixed Cost (S) & Fixed Cost (M) & Adap. \\
\midrule
MNIST     &  97.2\% & 97.2\% & 97.9\% & \textbf{98.6\%} \\
CIFAR-100 &  55.2\% & 55.8\% & 56.0\% & \textbf{60.1\%} \\
Caltech-101 &  76.2\% & 77.1\% & 77.7\% & \textbf{83.0\%} \\
MIT-67 & 51.6\% & 50.9\% & 49.7\% &  \textbf{57.0\%} \\
DIL & 70.0\% & 69.5\% & 69.3\% &  \textbf{72.6\%} \\
MLC & 66.3\% & 66.8\% & 65.7\% &  \textbf{68.6\%} \\
\bottomrule[.4mm]
\vspace{0.1mm}
\end{tabular}
}
\caption{{Comparisons of our approach (adaptive costs) with the fixed class-specific costs. The experimental protocols used for each dataset are shown in Fig.~\ref{fig:exp_set}. Fixed costs do not show a significant and consistent improvement in results.}}
\label{tab:comp_fixedcost}\vspace{-0.2cm}
\end{table}

{
Since our approach updates the costs with respect to the data statistics (i.e., data distribution, class separability and classification errors), 
{an interesting aspect is to analyse the performance when the costs are fixed and set equal to these statistics instead of updating them adaptively}. 
We experiment with fixed costs instead of adaptive costs in the {case of} CoSen-CNN. For this purpose, we used three versions of fixed costs, based on the class representation (H), data separability (S) and classification errors (M). 
Table~\ref{tab:comp_fixedcost} shows the results for each dataset with four different types of costs. 
{The results} show that none of the fixed costs significantly improve the performance in comparison to the adaptive cost.
This shows that the optimal costs are not the H, S and M themselves, rather an intermediate set of values give the best performance for cost-sensitive learning. }

\rev{Lastly, we observed  a smooth reduction in training and validation error for the case of cost-sensitive CNN. We show a comparison of classification errors between baseline and cost-sensitive CNNs at different training epochs in Fig.~\ref{fig:CIFAR_error}.}

\begin{figure}
\centering
\includegraphics[width = 0.85\columnwidth]{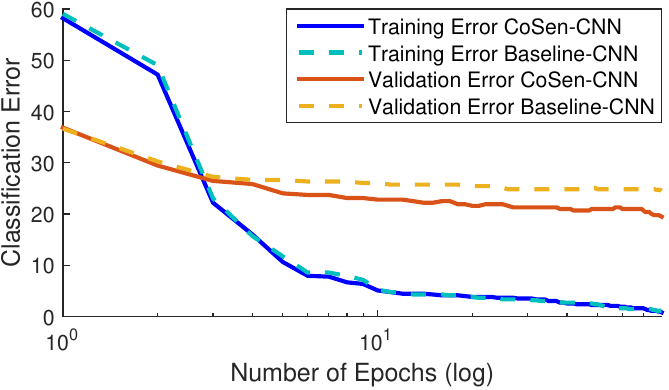} \\
\vspace{1mm}
\caption{ An observed decrease in the training and validation error on the DIL dataset (stand. split, Exp.~2) for the cases of {the} baseline and cost-sensitive CNNs.}\vspace{-0.2cm}
\label{fig:CIFAR_error}
\end{figure}

\vspace{0.2cm}
\noindent
\textbf{Timing Comparisons:}
The introduction of the \rev{class-dependent} costs did not prove to be prohibitive during the training of the CNN. 
For example, on an Intel quad core i7-4770 CPU (3.4GHz) with 32Gb RAM and Nvidia GeForce GTX 660 card (2GB), it took $80.19$ secs and $71.87$ secs to run one epoch with and without class sensitive parameters, respectively for the MIT-67 dataset.
At test time, the CoSen CNN took the same amount of time as that of the baseline CNN, because no extra computations were involved during testing.

\section{Conclusion}\label{sec:conclusion}
{ We proposed a \rev{cost-sensitive} deep CNN to deal with the class-imbalance problem, {which is commonly encountered when dealing with} \rev{real-world} datasets.
Our {approach}
is able to automatically set the \rev{class-dependent} costs based on the data statistics 
{of} the training set. 
We analysed three commonly used cost functions and introduced \rev{class-dependent} costs for each case. 
We show that the cost-sensitive CE loss function is c-calibrated and guess aversive. 
Furthermore, we proposed an alternating optimization procedure to efficiently learn the \rev{class-dependent} costs as well as the network parameters. 
Our results on six popular classification datasets show that the modified cost functions perform very well on the majority as well as on the minority classes in the dataset.}


\ifCLASSOPTIONcompsoc

  \section*{Acknowledgments}
  \small{This research was supported by an IPRS awarded by The University of Western Australia and Australian Research Council (ARC) grants
  {DP150100294} and DE120102960. 
  }
\else
 \section*{Acknowledgment}
  \small{This research was supported by an IPRS awarded by The University of Western Australia and Australian Research Council (ARC) grants
  {DP150100294} and DE120102960. 
  }
\fi

\appendices
\section{Proofs Regarding Cost Matrix $\xi'$}\label{sec:app}
\begin{lemma}
Offsetting the columns of the cost matrix $\xi'$ by any constant `$c$' does not affect the associated classification risk $\mathcal{R}$.  
\end{lemma}
\begin{proof}
From Eq.~1, we have: 
$$
\sum\limits_{q} \xi'_{p^*,q}P(q|\mathbf{x})  \leq \sum\limits_{q} \xi'_{p,q}P(q|\mathbf{x}) 
\quad \forall p \neq p*
$$
{which gives the following} relation:
\begin{align*}
P(p^*|\mathbf{x}) \left(\xi'_{p^*, p^*} - \xi'_{p, p^*}\right)  \leq  \\ 
\sum\limits_{q \neq p^*} P(q|\mathbf{x})\left(\xi'_{p,q} - \xi'_{p^*,q}\right), 
\quad \forall p \neq p*
\end{align*}
As indicated {in Sec. 3.1, }the above expression holds for all $p \neq p*$. 
{For a total number of $N$ classes and an optimal prediction $p^*$, there are $N-1$ of the above relations. }
By adding up the left and the right hand sides of these $N-1$ relations we get: 
\begin{align*}
P(p^*|\mathbf{x})\left((N-1)\xi'_{p^*, p^*} - \sum\limits_{p \neq p^*}\xi'_{p, p^*} \right) \leq \\
\sum\limits_{q \neq p^*} P(q|\mathbf{x})\left(\sum\limits_{p \neq p^*}\xi'_{p,q} - (N-1)\xi'_{p^*,q}\right), 
\end{align*}
This can be simplified to:
\begin{align*}
\mathbf{P}_{\mathbf{x}} \begin{bmatrix} 
 \sum_{i} \xi'_{i,1} - N\xi'_{p^*, 1}\\
 \vdots \\
 \sum_{i} \xi'_{i,N} - N\xi'_{p^*, N}
\end{bmatrix} 
\geq 0,
\end{align*}
where, $\mathbf{P}_{\mathbf{x}} = [P(1|\mathbf{x}), \ldots, P(N|\mathbf{x})]$. 
{Note that the posterior probabilities $\mathbf{P}_{\mathbf{x}}$ are positive ($\sum_i P(i|\mathbf{x}) = 1$ and $P(i|\mathbf{x}) > 0$).}
It can be seen from the above equation that the addition of any constant $c$, does not affect the overall relation, i.e., for any column $j$,
$$
\sum_{i} (\xi'_{i,j} + c) - N(\xi'_{p^*, j} + c) = \sum_{i} \xi'_{i,j} - N\xi'_{p^*, j}
$$
Therefore, the columns of the cost matrix can be shifted by a constant $c$ without any effect on the associated risk. 
\end{proof}

\begin{lemma}
The cost of the true class should be less than the mean cost of all misclassification.
\end{lemma}
\begin{proof}
Since, $\mathbf{P}_{\mathbf{x}}$ can take any distribution of values, we end up with the following constraint:
\begin{align*}
\sum_{i} \xi'_{i,j} - N\xi'_{p^*, j} \geq 0, \quad j \in [1,N].
\end{align*}
For a correct prediction $p^*$, $P(p^*|\mathbf{x}) > P(p|\mathbf{x}), \forall p \neq p^*$. 
Which implies that: 
\begin{align*}
\xi'_{p^*, p^*} \leq \frac{1}{N} \sum_{i} \xi'_{i,p^*}.
\end{align*}
It can be seen that the cost insensitive matrix (when $\text{diag(}\xi'\text{)} = 0$ and $\xi'_{i,j} = 1, \forall j \neq i$) satisfies this relation and provides the upper bound.
\end{proof}

\begin{lemma}\label{lem:three}
 The cost matrix $\xi$ for a cost-insensitive loss function is an all-ones matrix, $\mathbf{1}^{p\times p}$, rather than a $\mathbf{1}-\mathbf{I}$ matrix,  
{as in} the case of the traditionally used cost matrix $\xi'$. 
 \end{lemma}
 \begin{proof}
 With all costs equal to the multiplicative identity i.e., $\xi_{p,q} = 1$, the CNN activations will remain unchanged. Therefore, all decisions have a uniform cost of $1$ and the classifier is cost-insensitive. 
 \end{proof}
 \begin{lemma}\label{lem:four}
 All costs in $\xi$ are positive, i.e., $\xi \succ 0$. 
 \end{lemma}
 \begin{proof}
 {We adopt a proof by contradiction.}
 Let us suppose that $\xi_{p,q} = 0$. 
 During training {in this case, the} corresponding score for class $q$ ($s_{p,q}$) will always be zero for all samples belonging to class $p$. 
 {As a result,} the output activation ($y_{q}$) and {the} back-propagated error will be independent of the weight parameters of the network, which proves the Lemma.
 \end{proof}
 \begin{lemma}
 The cost matrix $\xi$ is defined such that all {of its} elements in are within the range $(0,1]$, i.e., $ \xi_{p,q} \in (0,1]$.
 \end{lemma}
 \begin{proof}
 {Based on} Lemmas~\ref{lem:three} and \ref{lem:four}, it is trivial that the costs are with-in the range $(0,1]$.
 \end{proof}
 \begin{lemma}\label{lem:six}
 Offsetting the columns of {the} cost matrix $\xi$ can lead to an equally probable guess point. 
 \end{lemma}
 \begin{proof}
 {Let us} consider the case of a cost-insensitive loss function. {In this case, $\xi = \mathbf{1}$ (from Lemma~\ref{lem:three}).} 
 Offsetting all of its columns by a constant $c=1$ will lead to $\xi=\mathbf{0}$. For $\xi=\mathbf{0}$, {the} CNN outputs will be zero for any $\mathbf{o}^{(i)} \in \mathbb{R}^{N}$. Therefore, the classifier will make a random guess for classification. 
 \end{proof}

\bibliographystyle{IEEEtran}
{\small \bibliography{egbib}}

\ifCLASSOPTIONcaptionsoff
  \newpage
\fi



\end{document}